\documentclass[lettersize,journal]{IEEEtran}

\usepackage{amsmath,bm,bbm,amsfonts,amssymb}
\usepackage{pifont}

\newcommand{\cmark}{\ding{51}}
\newcommand{\xmark}{\ding{55}}

\usepackage{soul,colortbl}
\usepackage[table,dvipsnames]{xcolor} 
\usepackage{color}

\usepackage{multirow}
\usepackage{tabularx}
\usepackage{booktabs}
\usepackage{pict2e}
\usepackage{url}
\usepackage{etoolbox}
\usepackage{orcidlink}
\usepackage{fontawesome5}
\usepackage{overpic}
\definecolor{cvprblue}{rgb}{0.21,0.49,0.74}
\definecolor{LightCyan}{rgb}{0.88,1,1}
\definecolor{gcolor}{RGB}{40,160,70}
\definecolor{ycolor}{RGB}{222,158,20}
\usepackage{hyperref}
\hypersetup{
     colorlinks=true,
     linkcolor=blue,
     citecolor=blue,
     filecolor=blue,
     urlcolor=magenta,
}
\usepackage{cite}

\usepackage{makecell}
\usepackage{siunitx} \usepackage[linesnumbered,ruled,vlined]{algorithm2e}
\usepackage[capitalise]{cleveref}
\Crefname{section}{Sec.}{Sects.}
\crefname{figure}{Fig.}{Figs.}
\Crefname{figure}{Fig.}{Figs.}
\crefname{table}{Tab.}{Tabs.}
\Crefname{table}{Tab.}{Tabs.}
\crefname{appendix}{}{Appendixes}
\Crefname{appendix}{}{Appendixes}
\crefname{proposition}{Proposition}{Propositions}
\Crefname{proposition}{Proposition}{Propositions}

\crefformat{equation}{Eq.~{#2(#1)#3}}

\newtheorem{proposition}{Proposition}

\usepackage{subcaption}
\usepackage{arydshln}
\usepackage{tikz}

\NewDocumentCommand{\emojiBR}{ O{blue} O{\normalsize} m }{\put(1,4){\color{#1}#2 #3}}

\usepackage{algorithmic}

\usepackage{contour}
\contourlength{0.1em} 

\newcommand{\myPara}[1]{\textbf{#1}.}

\definecolor{currentOrange}{RGB}{255, 165, 0} 
\definecolor{deviatedBlue}{RGB}{0, 100, 200} 
\definecolor{sessionGreen}{RGB}{0, 150, 0} 
\definecolor{oldYellow}{RGB}{255, 204, 0} 
\definecolor{scoreRed}{RGB}{255, 0, 0} 

\definecolor{rowgrayblue}{RGB}{245,248,250}
\definecolor{rowgraygreen}{RGB}{246,249,246}
\definecolor{rowgrayyellow}{RGB}{250,249,243}
\definecolor{rowgraypink}{RGB}{250,246,247}

\newcommand{\std}[1]{{\normalfont\color{gray}\ensuremath{\,\pm\,}#1}}

\begin{document}

\title{MePo++: Unifying Representation Refinement and Reconciliation for General Continual Learning}

\author{
    Guanglong Sun$^{\orcidlink{0009-0001-8403-1925}}$, Kanglei Zhou$^{\orcidlink{0000-0002-4660-581X}}$, Liyuan Wang$^{\orcidlink{0009-0002-7797-325X}}$,~\IEEEmembership{Member,~IEEE}, Qi Cheng, Hongwei Yan$^{\orcidlink{0009-0009-6174-8390}}$, Shuang Cui$^{\orcidlink{0000-0001-9293-7316}}$, \\ Hang Su$^{\orcidlink{0000-0001-8294-6315}}$, Jun Zhu$^{\orcidlink{0000-0002-6254-2388}}$,~\IEEEmembership{Fellow,~IEEE}, and Yi Zhong$^{\orcidlink{0000-0002-7927-5976}}$
    \thanks{
        Manuscript received \today.
        This work was supported by the NSFC Project Nos.~T2622023, 62406160, the Beijing Natural Science Foundation L247011, and the Beijing Major Science and Technology Project No. Z251100008425003. 
        \textit{(Guanglong Sun and Kanglei Zhou contributed equally to this work.)}
        \textit{(Corresponding author: Liyuan Wang and Yi Zhong).}
    }\thanks{
        Guanglong Sun, Hongwei Yan and Yi Zhong are with the School of Life Sciences, IDG/McGovern Institute for Brain Research, Tsinghua University, Beijing, China (e-mail: sgl23@mails.tsinghua.edu.cn; 
        yanhw22@mails.tsinghua.edu.cn;
        zhongyithu@tsinghua.edu.cn).
        
        Kanglei Zhou, Qi Cheng, and Liyuan Wang are with the Department of Psychological and Cognitive Sciences, Tsinghua University, Beijing, China (e-mail: zhoukanglei@tsinghua.edu.cn; liyuanwang@tsinghua.edu.cn).
        
        Shuang Cui is with the Institute of Software, Chinese Academy of Sciences, Beijing, China (e-mail: cuishuang21@mails.ucas.ac.cn).
        
        Hang Su and Jun Zhu are with the Department of Computer Science and Technology, Institute for AI, BNRist Center, THBI Lab, Tsinghua-Bosch Joint Center for ML, Tsinghua University, Beijing, China (email: suhangss@mail.tsinghua.edu.cn; dcszj@tsinghua.edu.cn).
    }
}
\markboth{Journal of \LaTeX~Class Files,~Vol.~XX, No.~XX, XX~XXXX}{Shell \MakeLowercase{\textit{et al.}}: A Sample Article Using IEEEtran.cls for IEEE Journals}

\maketitle

\begin{abstract}
General continual learning (GCL) aims to learn from evolving data streams without task identities, explicit boundaries, or repeated access to previous data, making it a realistic yet challenging setting for continual intelligence. Although pretrained models (PTMs) provide rich prior knowledge for addressing the limited supervision and non-stationary nature of GCL, existing PTM-based methods often directly adapt pretrained representations and overlook two critical gaps: the misalignment between upstream pretraining and downstream continual adaptation, and the unreliability of conventional output alignment under blurry streams. Here we propose \textbf{MePo++}, a unified post-training framework that bridges pretrained knowledge and downstream GCL through representation refinement and reconciliation. MePo++ introduces two complementary components: \textbf{MetaPrep}, which improves representation plasticity for continual adaptation through unsupervised meta-refinement over pseudo continual sequences; and \textbf{StreamAlign}, which reinforces representation stability by reconciling evolving online features with a stable pretrained geometry. By improving representation learnability before adaptation and preserving alignment during continual learning, MePo++ enables PTMs to remain both plastic for new concepts and stable over evolving streams. Experiments across diverse PTMs, datasets, and continual learning baselines demonstrate the consistent effectiveness and generality of MePo++ for PTM-based GCL. Our code is available at \url{https://github.com/SunGL001/MePo_Plus}.
\end{abstract}


\section{Introduction}
\IEEEPARstart{H}{uman} intelligence can continuously acquire new knowledge from evolving experiences while retaining and reusing previously learned information. Inspired by this capability, continual learning (CL) aims to develop artificial systems that adapt to changing data distributions without catastrophically forgetting earlier knowledge~\cite{wang2024comprehensive,van2019three}. However, conventional CL assumes clearly separated tasks, explicit boundaries, and task identities, which simplify learning and deviate from real-world environments. General continual learning (GCL) offers a more realistic formulation with online streams, blurry transitions, temporally mixed concepts, unavailable task identities, and limited memory~\cite{buzzega2020dark,de2021continual,moon2023online,kang2025advancing}. These conditions make GCL harder and, in turn, require rapid adaptation, long-term retention, and robust learning as concepts and distributions evolve over time during deployment.

\begin{figure}
    \centering
    \includegraphics[width=\linewidth,clip,trim=230 150 225 150]{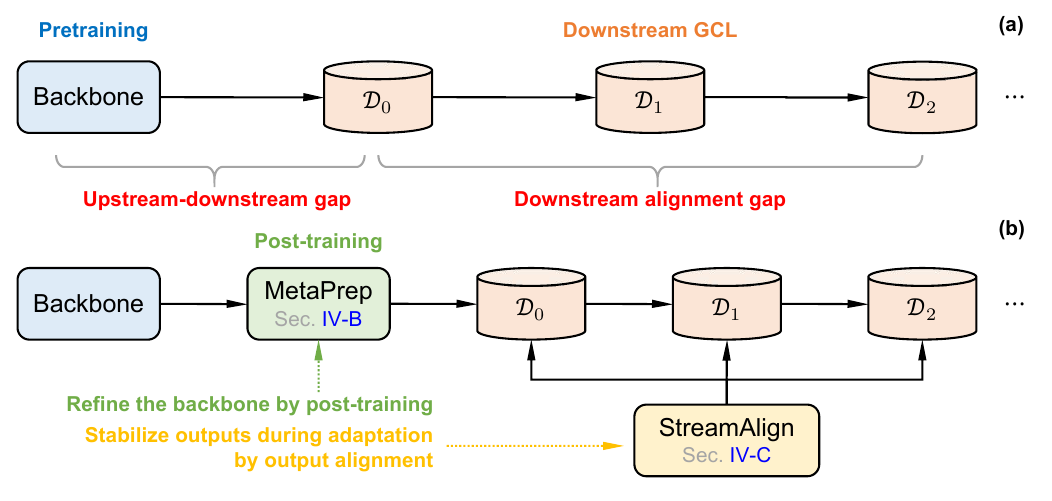}
    \begin{overpic}[width=\linewidth]{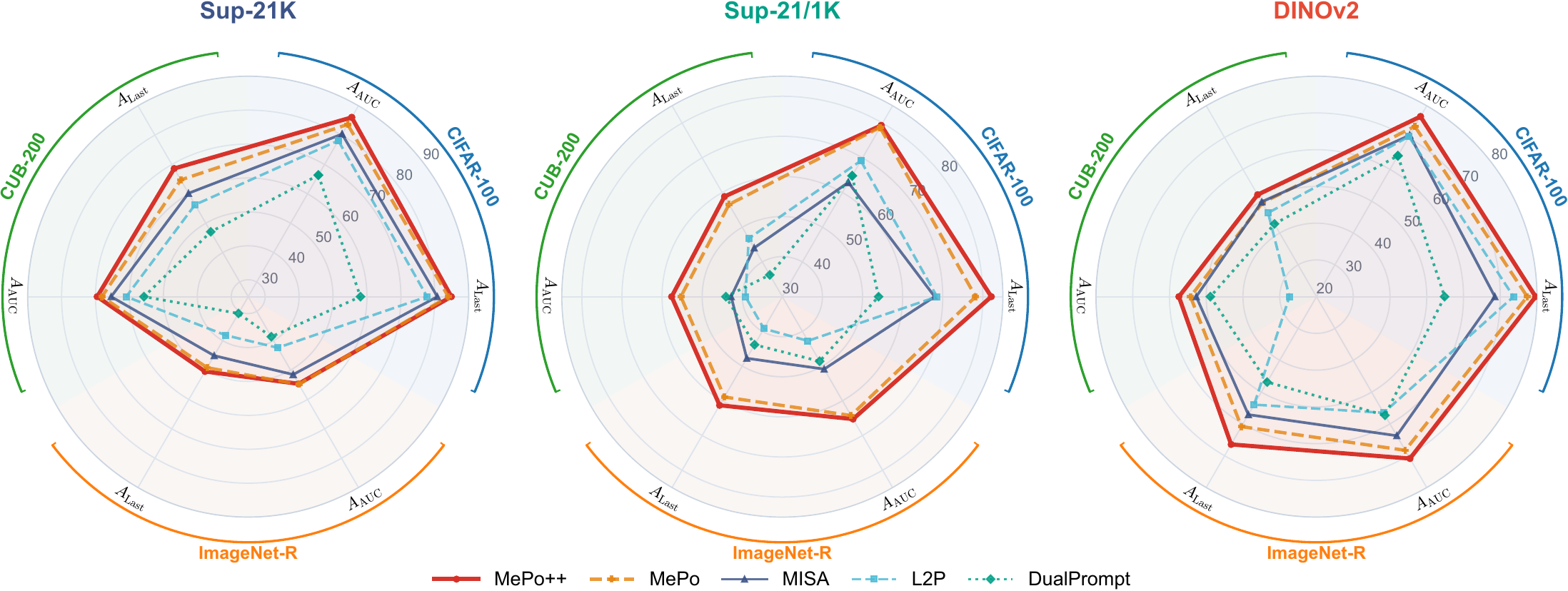}
        \put(0,36){\sf\tiny\textbf{(c)}}
        \put(32,36){\sf\tiny\textbf{(d)}}
        \put(64,36){\sf\tiny\textbf{(e)}}
    \end{overpic}
    \caption{
    \textbf{Motivation and performance overview.}
    (a) Existing PTM-based GCL directly adapts pretrained representations to evolving data streams, leading to an upstream-downstream gap and a downstream alignment gap.
    (b) MePo++ addresses the two gaps through MetaPrep, which refines pretrained representations before deployment, and StreamAlign, which reconciles evolving representations during continual adaptation.
    (c--e) Performance comparisons under Sup-21K, Sup-21/1K, and DINOv2 pretraining across CIFAR-100, ImageNet-R, and CUB-200, demonstrating consistent improvements over representative GCL baselines across all evaluated settings.
    }
    \label{fig:teaser}
    \phantomsubcaption\label{fig:teaser-a}
    \phantomsubcaption\label{fig:teaser-b}
    \phantomsubcaption\label{fig:teaser-c}
    \phantomsubcaption\label{fig:teaser-d}
    \phantomsubcaption\label{fig:teaser-e}
\end{figure}

Meeting these requirements remains difficult for existing CL methods. They address forgetting through synaptic regularization \cite{kirkpatrick2017overcoming}, architectural isolation \cite{rebuffi2017icarl}, or replay \cite{buzzega2020dark}, but most assume task-incremental protocols with explicit boundaries and ample task-wise data. Consequently, they transfer poorly to temporally mixed, single-pass streams \cite{wang2024comprehensive}. Dedicated GCL methods \cite{kang2025advancing,yan2026flyprompt} relax these assumptions but often train representations from scratch and rely on replay, reducing sample efficiency and adding storage and privacy costs. PTMs offer a promising alternative by providing transferable prior knowledge before deployment \cite{wang2024hide}. This raises a central research question: \textbf{How can PTM knowledge support realistic GCL?}

We therefore examine PTM-based GCL and identify two fundamental gaps (see \cref{fig:teaser-a}). First, strong upstream pretraining does not guarantee effective continual adaptation. Existing methods usually inherit pretrained representations and optimize only lightweight modules, yet sparse, single-pass, temporally mixed observations can still cause limited acquisition or severe interference. This \textbf{upstream-downstream misalignment} means generic representations are insufficiently prepared for rapid, stable adaptation. Second, output alignment becomes unreliable when task boundaries are blurry: task-wise feature distributions or class statistics are incomplete and biased as old and new concepts mix. This \textbf{downstream alignment gap} requires maintaining a coherent representation and decision space without explicit boundaries or reliable task statistics. Together, the gaps impose a plasticity--stability requirement: representations must absorb new concepts while resisting harmful drift. This requirement echoes biological learning, where prior experience regulates neural plasticity and new memories integrate with persistent structure rather than overwrite it as experience accumulates over time~\cite{abraham2008metaplasticity,richards2017persistence,lei2022social,lei2024reconstructing}.

These observations motivate \textbf{MePo++}, a unified post-training framework with two complementary modules (see \cref{fig:teaser-b}). \textbf{MetaPrep} addresses the \textbf{upstream-downstream misalignment} before deployment: it clusters features from unlabeled meta data, arranges the pseudo-classes into continual sequences, and applies bi-level meta-learning that simulates sequential adaptation and optimizes cross-sequence generalization. The result is an initialization better suited to acquiring new concepts while reducing destructive interference from sparse, mixed updates. \textbf{StreamAlign} addresses the \textbf{downstream alignment gap} during deployment without task boundaries or task-wise statistics. It uses the refined representation geometry as a stable reference, reconstructs each transient online representation toward this geometry, and applies supervised contrastive reconciliation to keep same-class features close and different-class features separate. Together, MetaPrep prepares representations before deployment, while StreamAlign preserves stable and discriminative features throughout the evolving data stream under non-stationary online conditions.

Extensive experiments demonstrate the consistent effectiveness and generality of MePo++ across diverse PTMs and GCL settings. As shown in \cref{fig:teaser-c,fig:teaser-d,fig:teaser-e}, MePo++ surpasses the strongest baselines in all 18 combinations of three PTMs, three datasets, and two CL metrics, with an average gain of 7.93 percentage points (14.10\% relative). The average improvements reach 5.01, 12.88, and 5.90 percentage points for Sup-21K, Sup-21/1K, and DINOv2, respectively, and 7.57, 8.70, and 7.53 points on CIFAR-100, ImageNet-R, and CUB-200, respectively, with relative gains up to 30.69\%. Under the more challenging few-shot GCL setting with only 20\% of the training data, MePo++ also consistently outperforms strong baselines across different PTMs and downstream learners, further demonstrating its robustness under limited observations.

Our preliminary version, MePo (Meta Post-refinement)~\cite{sun2026mepo}, pioneered the exploration of post-training pretrained representations for GCL. MePo++ substantially extends MePo in three aspects. First, MetaPrep replaces label-dependent sequence construction with unsupervised clustering over pretrained features, enabling scalable post-training on large-scale unlabeled data. Second, StreamAlign extends covariance-based feature interpolation into semantic representation reconciliation through supervised contrastive learning, preserving both geometric stability and class discriminability. Third, we broaden the evaluation beyond standard GCL to few-shot GCL and CLIP-based continual learning, examining the framework under both limited downstream supervision and vision-language pretraining. Together, these extensions transform MePo from a label-dependent refinement approach into a more general post-training framework for PTM-based GCL across diverse deployment conditions.

Our main contributions are summarized as follows:
\begin{itemize}
\item We identify two bottlenecks in PTM-based GCL: pretrained representations lack continual preparedness, and output alignment becomes unreliable without task boundaries or complete statistics. These gaps expose the need to balance plasticity and stability under sparse, mixed, single-pass online observations in realistic deployments.

\item We propose \textbf{MePo++}, a unified post-training framework that prepares pretrained representations before deployment and aligns them throughout the stream. It requires no task identities, supports diverse backbones and learners, and operates under single-pass incoming streams with evolving concepts and distributions.

\item We develop two complementary modules: \textbf{MetaPrep} clusters unlabeled features and applies bi-level meta-refinement to build a GCL-ready initialization; \textbf{StreamAlign} aligns online features to a stable geometry prior while preserving semantic consistency. Together, they balance rapid acquisition and stable discrimination under sparse, mixed, single-pass streams.

\item We demonstrate consistent gains across standard, few-shot, and CLIP-based continual learning with diverse PTMs and learners. Comparisons, ablations, sensitivity tests, and representation analyses validate the complementary roles of MetaPrep and StreamAlign across diverse continual-learning settings and datasets.

\end{itemize}

\section{Related Work}

\myPara{General Continual Learning}
Continual learning aims to enable models to continuously acquire new knowledge while preserving previously learned information, yet limited by catastrophic forgetting. Regularization-based methods constrain parameter updates by preserving important model components, such as EWC~\cite{kirkpatrick2017overcoming}, LwF~\cite{Li17learning}, and synaptic intelligence~\cite{Zenke17}. 
Replay-based methods maintain representative historical samples or generated data to alleviate forgetting, including experience replay~\cite{rolnick2019experience}, iCaRL~\cite{rebuffi2017icarl}, and dark experience replay~\cite{buzzega2020dark}.
Although effective, most conventional CL methods assume explicit task boundaries or task identities, which are rarely available in real-world scenarios.
GCL extends CL toward more realistic online environments with blurry task transitions, unknown task identities, and limited memory resources~\cite{wang2024comprehensive}.
Under such single-pass streams, a line of work focuses on which samples to retain and how to schedule them, including gradient-based sample selection~\cite{aljundi2019gradient}, diversity-aware memory construction~\cite{bang2021rainbow}, importance-based online sampling~\cite{koh2021online}, and coordinated replay scheduling~\cite{yan2024orchestrate}.
Recent GCL methods investigate efficient adaptation under continuous data streams through experience replay~\cite{buzzega2020dark,fini2020online}, prompt-based adaptation~\cite{moon2023online}, and dynamic routing mechanisms~\cite{yan2026flyprompt}.
However, existing GCL methods mainly optimize the downstream learner after deployment, leaving the continual suitability of pretrained representations largely unexamined under sparse, blurry online streams.

\myPara{Pretrained Models for Continual Learning}
Pretrained models provide transferable representations that support parameter-efficient CL through prompts and adapters. Prompt-based methods such as L2P~\cite{wang2022learning} and DualPrompt~\cite{wang2022dualprompt} selectively activate pretrained knowledge while limiting updates, and later work improves prompt optimization~\cite{wang2024hierarchical}. A complementary direction keeps the backbone frozen and calibrates the output space instead, for example through random projection with class-wise second-order statistics~\cite{mcdonnell2024ranpac}. Beyond vision-only encoders, vision-language models have also been adapted for CL through gradient-aware modulation~\cite{huang2025mind}, mixture-of-expert adapters~\cite{yu2024moeadapters}, probabilistic prompt adaptation~\cite{jha2024clap4clip}, and online low-rank updates~\cite{wei2025onlinelora}. However, these methods generally assume that pretrained representations are already suitable for CL. By optimizing only in-stream adaptation, they overlook whether the initial representation can learn continually from sparse, temporally mixed observations, motivating dedicated post-training before deployment.

\myPara{Representation Learning and Post-training Adaptation}
Large-scale self-supervised learning produces transferable visual features without extensive annotations~\cite{zhou2026comprehensive}, using contrastive methods such as SimCLR~\cite{chen2020simple} and MoCo~\cite{he2020momentum}, clustering methods such as DeepCluster~\cite{caron2018deep} and SwAV~\cite{caron2020unsupervised}, and masked objectives~\cite{he2022masked}. These advances highlight the importance of representation quality for downstream adaptation. Continual self-supervised methods study this connection~\cite{rao2019continual,caccia2021special,fini2022self}, but learn representations jointly with continual updates and therefore do not directly exploit powerful PTMs in evolving GCL streams. MePo++ instead introduces dedicated post-training: MetaPrep improves continual learnability before deployment, and StreamAlign maintains stability online without task boundaries or task-level statistics.

\section{Empirical Analysis}
\label{sec:empirical}

This section formulates PTM-based GCL, tests adaptation and output alignment under evolving blurry streams, and derives design principles for practical continual deployment.

\subsection{Problem Formulation}
\label{subsec:problem_formulation}

\myPara{General Continual Learning}
We consider the GCL setting, where an agent learns
from an online and continuously evolving data stream. Let
$\mathcal{D}=\{\mathcal{B}_{t}\}_{t=1}^{T}$ denote the entire data stream, where
$T$ is the total number of learning steps and $\mathcal{B}_{t}=\{(\mathbf{x}_{i},y_{i})\}_{i=1}^{N_t}$ denotes the mini-batch observed at step $t$. Here,
$N_t$ is the number of samples in the current mini-batch,
$\mathbf{x}_{i}\in\mathcal{X}$ represents an input sample, and
$y_i$ denotes its corresponding label. Unlike conventional task-incremental
CL, GCL assumes that explicit task identities and task
boundaries are unavailable. Instead, previously observed and newly emerging
concepts may coexist within the same mini-batch, resulting in continuously
evolving and blurry data streams. Each sample is processed only once,
requiring the learner to acquire new knowledge while preserving previously
learned information from past observations.

\myPara{PTM-based GCL}
We consider a pretrained encoder
$f_{\boldsymbol{\theta}_{0}}:\mathcal{X}\rightarrow\mathbb{R}^{d}$, where
$\boldsymbol{\theta}_{0}$ denotes the pretrained parameters and $d$ is the
feature dimension. The pretrained model provides an initial representation
before deployment, which is further adapted by a downstream GCL learner
$\mathcal{A}$. At step $t$, the learner updates the model parameters according
to $(\boldsymbol{\theta}_{t},\boldsymbol{\phi}_{t})=\mathcal{A}(\boldsymbol{\theta}_{t-1},\boldsymbol{\phi}_{t-1},\mathcal{B}_{t})$, where
$\boldsymbol{\theta}_{t}$ denotes the encoder parameters after adaptation and
$\boldsymbol{\phi}_{t}$ represents the downstream parameters, such as the
classifier or task-specific components. The resulting predictor at step $t$ is
defined as $h_t(\mathbf{x})=g_{\boldsymbol{\phi}_{t}}(f_{\boldsymbol{\theta}_{t}}(\mathbf{x}))$, where
$g_{\boldsymbol{\phi}_{t}}$ denotes the prediction head parameterized by
$\boldsymbol{\phi}_{t}$.

\myPara{Challenges in PTM-based GCL}
Although pretrained models provide strong transferable representations, their
pretraining objectives are generally independent of downstream GCL streams.
This discrepancy introduces two challenges specific to PTM-based GCL. First,
the upstream representations optimized for static recognition may not possess
the continual adaptability required by evolving streams, resulting in an
upstream-downstream misalignment. Second, maintaining a stable representation
space during online adaptation is challenging, as conventional output
alignment strategies often rely on task-dependent statistics that become
unreliable under blurry streams. These challenges motivate our empirical
investigation from two complementary perspectives: whether existing
representation adaptation strategies can improve the continual learnability of
pretrained representations, and whether conventional alignment strategies can
maintain reliable representations under realistic GCL conditions.

\begin{figure}
    \centering
    \includegraphics[width=\linewidth]{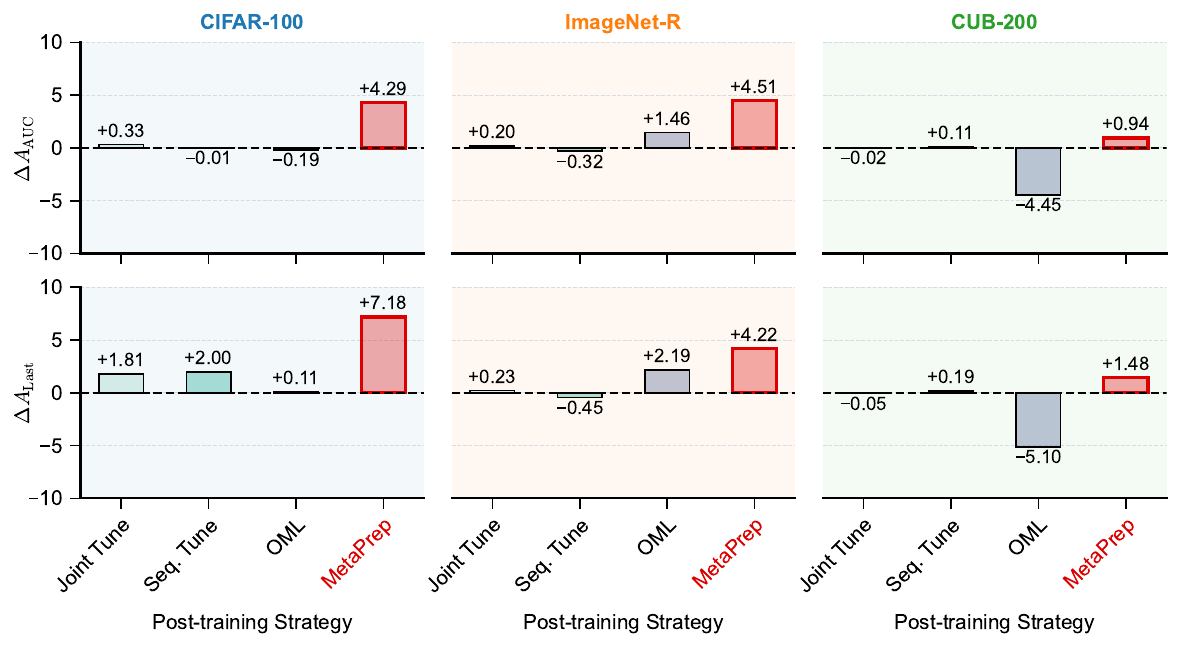}
    \caption{
    \textbf{Empirical analysis of post-training strategies for improving continual learnability.}
    We report the performance changes over the original DualPrompt with a Sup-21K pretrained backbone, where the baseline results are provided in \cref{tab:main}.
    Results on CIFAR-100, ImageNet-R, and CUB-200 show the effects of different post-training strategies on GCL performance.
    Positive and negative values indicate improvements and degradations compared with the original baseline, respectively.
    While existing post-training strategies for downstream adaptation yield limited or unstable gains, MetaPrep consistently improves both $A_{\mathrm{AUC}}$ and $A_{\mathrm{Last}}$, suggesting that pretrained representations benefit from dedicated optimization for continual adaptation.
    }
    \label{fig:empirical-post}
\end{figure}

\subsection{Continual Learnability of Pretrained Representations}
\label{subsec:continual_learnability}

A key challenge in PTM-based GCL is whether pretrained representations adapt effectively to evolving streams. Large-scale pretraining targets static recognition, whereas post-training for continual learnability remains underexplored. We therefore ask whether post-training can make pretrained representations suitable for evolving streams.

We compare three alternatives with increasing continual-learning awareness: \emph{Joint Tune} adapts on aggregated pretraining data, \emph{Seq. Tune} adapts sequentially to mimic an evolving stream, and \emph{OML} uses an online meta-learning objective~\cite{javed2019meta}. All use the same Sup-21K backbone, with changes measured against the original DualPrompt baseline~\cite{wang2022dualprompt}.
As shown in \cref{fig:empirical-post}, existing strategies provide limited and inconsistent improvements across CIFAR-100, ImageNet-R, and CUB-200. On CIFAR-100, Joint Tune, Seq. Tune, and OML change $A_{\mathrm{AUC}}$ by only $+0.33\%$, $-0.01\%$, and $-0.19\%$, respectively, while their improvements in $A_{\mathrm{Last}}$ are $+1.81\%$, $+2.00\%$, and $+0.11\%$. The inconsistency becomes more evident across datasets: OML improves $A_{\mathrm{AUC}}$ by $1.46\%$ on ImageNet-R but decreases it by $4.45\%$ on CUB-200. In comparison, the GCL-oriented post-training strategy consistently yields larger gains, improving $A_{\mathrm{AUC}}$ by $+4.29\%$, $+4.51\%$, and $+0.94\%$ and $A_{\mathrm{Last}}$ by $+7.18\%$, $+4.22\%$, and $+1.48\%$ across the three datasets. These results reveal an upstream-downstream misalignment in PTM-based GCL: post-training is a promising means of adapting pretrained representations, but existing objectives do not consistently optimize the continual learnability required by evolving GCL streams.

\subsection{Representation Alignment under Evolving Streams}
\label{subsec:representation_alignment}

Beyond initialization, PTM-based GCL must maintain a stable representation space during online adaptation. Output alignment can correct drift, but existing strategies rely on class-wise statistics from the stream. Without task boundaries, old and new concepts coexist, making these estimates incomplete and unstable. We therefore ask whether conventional output alignment remains reliable under blurry, evolving streams.

We compare three mechanisms with the same DualPrompt baseline~\cite{wang2022dualprompt} and Sup-21K backbone: \emph{SLCA-Out} uses SLCA's classifier alignment~\cite{zhang2023slca}, \emph{MVP-Out} uses MVP's output alignment~\cite{moon2023online}, and \emph{MISA-Out} uses MISA's alignment component~\cite{kang2025advancing}. Each corrects predictions using statistics accumulated during adaptation.
As shown in \cref{fig:empirical-out}, conventional alignment mechanisms exhibit substantially different behavior under GCL. SLCA-Out consistently degrades performance, reducing $A_{\mathrm{AUC}}$ by $3.65\%$, $3.76\%$, and $4.98\%$ and $A_{\mathrm{Last}}$ by $6.97\%$, $3.26\%$, and $4.56\%$ on CIFAR-100, ImageNet-R, and CUB-200, respectively. MVP-Out and MISA-Out generally improve the baseline, but their gains remain dataset- and metric-dependent: MVP-Out improves $A_{\mathrm{Last}}$ by $4.28\%$ on CIFAR-100 but only $1.88\%$ on CUB-200, while MISA-Out yields $A_{\mathrm{Last}}$ gains ranging from $1.06\%$ on ImageNet-R to $3.65\%$ on CUB-200. In comparison, the GCL-oriented alignment strategy consistently achieves larger gains, improving $A_{\mathrm{AUC}}$ by $2.74\%$, $3.84\%$, and $1.78\%$ and $A_{\mathrm{Last}}$ by $7.67\%$, $6.21\%$, and $5.80\%$ across the three datasets. These observations reveal a downstream alignment gap: output alignment can be beneficial, but its effectiveness is highly sensitive to the reliability of the statistics available from blurry and evolving streams.

\begin{figure}
    \centering
    \includegraphics[width=\linewidth]{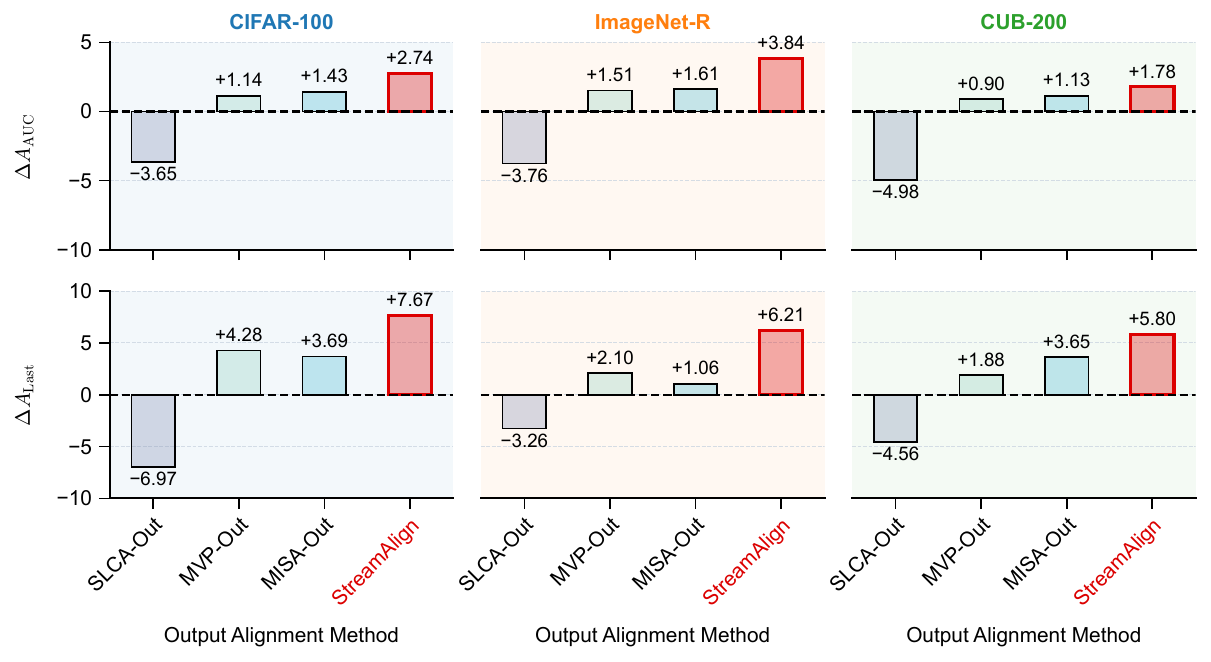}
    \caption{
    \textbf{Empirical analysis of output alignment strategies.}
    We report the performance improvement over the original DualPrompt baseline with a Sup-21K pretrained backbone, where the baseline results are provided in \cref{tab:main}.
    Results are shown for different output alignment methods on CIFAR-100, ImageNet-R, and CUB-200.
    While conventional alignment strategies provide limited and dataset-dependent improvements, StreamAlign consistently achieves stronger gains in both $A_{\mathrm{AUC}}$ and $A_{\mathrm{Last}}$, demonstrating the effectiveness of maintaining stable representation alignment under GCL.
    }
    \label{fig:empirical-out}
\end{figure}

\begin{figure*}
    \centering
    \includegraphics[width=\linewidth,clip,trim=10 90 0 50]{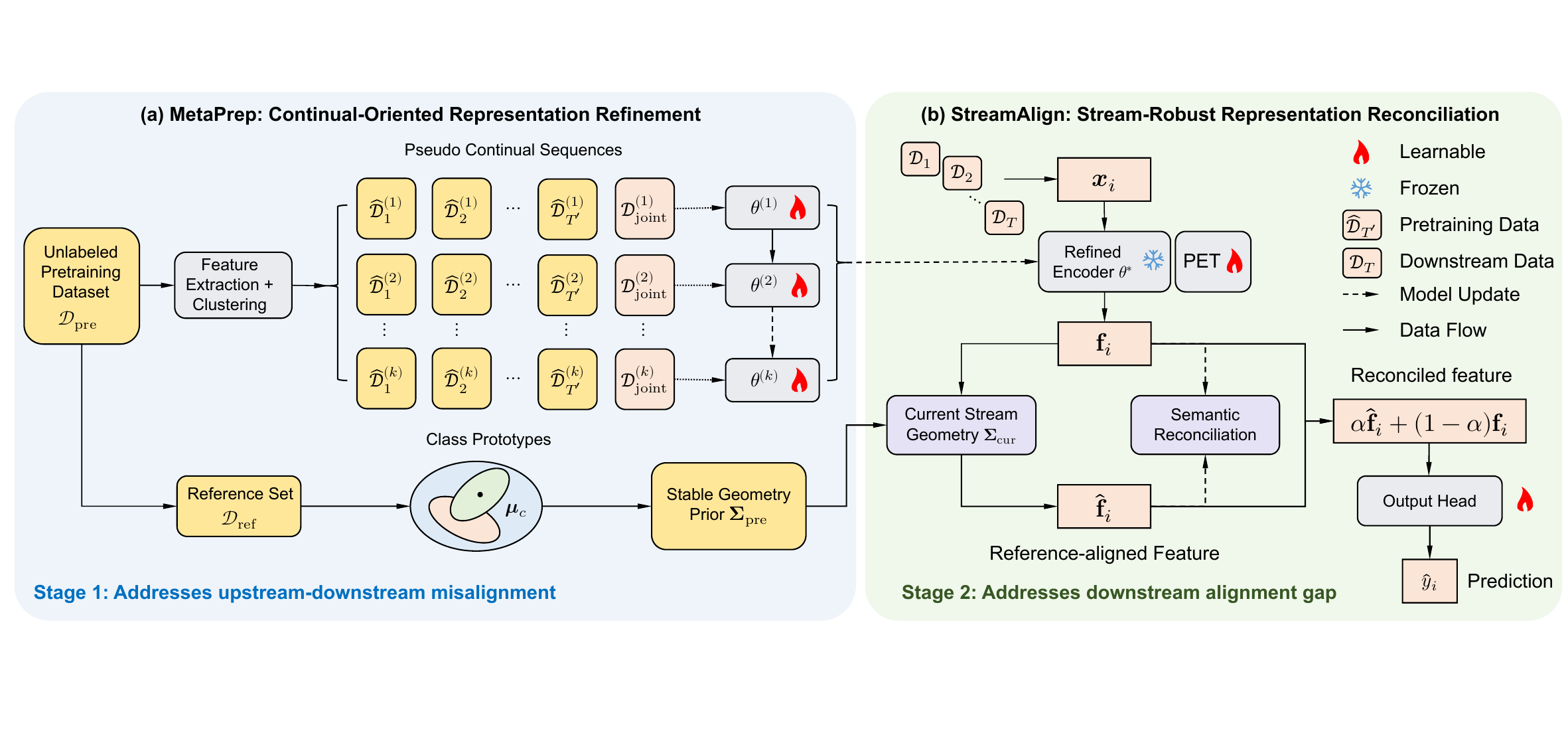}
    \caption{\textbf{Framework overview of MePo++.} MePo++ unifies representation refinement before deployment and representation reconciliation during continual adaptation. In Stage 1, MetaPrep constructs pseudo continual sequences from unlabeled pretraining data and performs bi-level meta-refinement to obtain a GCL-oriented encoder together with a stable representation geometry prior, addressing the upstream-downstream misalignment. In Stage 2, StreamAlign reconciles transient online representations with this stable prior through geometry-guided reconstruction and semantic regularization, maintaining reliable and discriminative representations under evolving GCL streams.}
    \label{fig:framework}
\end{figure*}

\subsection{Implications for Method Design}
\label{subsec:implications}

The findings imply two requirements for PTM-based GCL. First, pretrained representations should be prepared for continual adaptation so they can acquire emerging concepts while retaining transferable knowledge. Second, online alignment should rely on stable information rather than unreliable task-wise statistics from blurry streams. These requirements motivate a two-stage design: prepare representations before deployment, then align them online to retain stable geometry and clear class boundaries over time.

\section{MePo++: Unified Representation Refinement and Reconciliation for GCL}
\label{sec.method}

Motivated by \cref{subsec:implications}, MePo++ bridges upstream pretraining and downstream GCL in two stages: representation refinement prepares pretrained knowledge before deployment, and representation reconciliation preserves a stable, discriminative space during online learning. A theoretical interpretation of this lifecycle, including local update compatibility in MetaPrep and controlled geometric correction in StreamAlign, is provided in \cref{app:theory}. We next detail both stages under realistic, evolving stream conditions.

\subsection{Framework Overview}
\label{sec.framework}

MePo++ formulates GCL as a two-stage process of representation preparation and online stabilization. Before deployment, \textbf{MetaPrep} refines generic pretrained representations for rapid continual adaptation, thereby mitigating upstream-downstream misalignment. During deployment, \textbf{StreamAlign} regularizes transient online representations against a stable reference geometry, mitigating the downstream alignment gap while maintaining semantic discrimination. Together, the two stages promote pre-deployment plasticity and preserve stability throughout the evolving stream.

As illustrated in \cref{fig:framework}, let $\mathcal{D}_{\mathrm{pre}}=\{\mathbf{x}_i\}_{i=1}^{N}$ denote the unlabeled pretraining data and $f_{\boldsymbol{\theta}_0}:\mathcal{X}\rightarrow\mathbb{R}^{d}$ the original pretrained encoder. In Stage 1, MetaPrep organizes the latent semantic structure encoded by the PTM into pseudo continual sequences, converting static upstream data into evolving experiences without semantic annotations. Bi-level meta-refinement then optimizes the encoder for performance after simulated sequential updates, transforming $\boldsymbol{\theta}_0$ into a GCL-oriented initialization $\boldsymbol{\theta}^{\ast}$. MetaPrep further summarizes the refined representation space as a stable geometry prior $\boldsymbol{\Sigma}_{\mathrm{pre}}\in\mathbb{R}^{d\times d}$; thus, Stage 1 produces $(\boldsymbol{\theta}^{\ast},\boldsymbol{\Sigma}_{\mathrm{pre}})$ for downstream learning.

In Stage 2, $\boldsymbol{\theta}^{\ast}$ initializes the encoder for the GCL stream $\mathcal{D}=\{\mathcal{B}_t\}_{t=1}^{T}$. For each $\mathbf{x}_i\in\mathcal{B}_t$, the current encoder produces $\mathbf{f}_i=f_{\boldsymbol{\theta}_t}(\mathbf{x}_i)$, and the batch induces a transient geometry $\boldsymbol{\Sigma}_{\mathrm{cur}}$. Because sparse, mixed, and non-stationary observations can bias this local geometry, StreamAlign uses $\boldsymbol{\Sigma}_{\mathrm{pre}}$ to reconstruct a reference-aligned feature $\hat{\mathbf{f}}_i$ while retaining the semantics of the plastic online feature $\mathbf{f}_i$. Their reconciliation yields $\mathbf{f}_{\mathrm{trans},i}$ for prediction, preserving stream-specific adaptation within a stable representation structure. Operationally, MetaPrep derives $(\boldsymbol{\theta}^{\ast},\boldsymbol{\Sigma}_{\mathrm{pre}})$ before deployment (see \cref{sec.metaprep}), whereas StreamAlign maps $(\mathbf{f}_i,\boldsymbol{\Sigma}_{\mathrm{cur}},\boldsymbol{\Sigma}_{\mathrm{pre}})$ to $\mathbf{f}_{\mathrm{trans},i}$ online (see \cref{sec.streamalign}). This decomposition decouples pre-deployment preparation from online stabilization.

\subsection{MetaPrep: Continual-Oriented Representation Refinement}
\label{sec.metaprep}

The observations in \cref{subsec:continual_learnability} reveal an upstream--downstream misalignment: pretrained representations encode transferable knowledge but are not optimized for sparse, evolving observations. MetaPrep addresses this gap before deployment by converting unlabeled data into simulated continual experiences and optimizing the encoder after sequential adaptation. It (1) discovers latent structure and constructs pseudo sequences, (2) performs bi-level meta-refinement, and (3) extracts a stable geometric prior. As illustrated in \cref{fig:metaprep}, sequential adaptation on pseudo tasks is followed by joint evaluation; their parameter displacement updates the initialization toward parameters that remain effective after adaptation. MetaPrep outputs a GCL-oriented initialization $\boldsymbol{\theta}^{\ast}$ and prior $\boldsymbol{\Sigma}_{\mathrm{pre}}$ for downstream continual adaptation.

\begin{figure}
    \centering
    \includegraphics[width=0.7\linewidth,clip,trim=260 100 260 100]{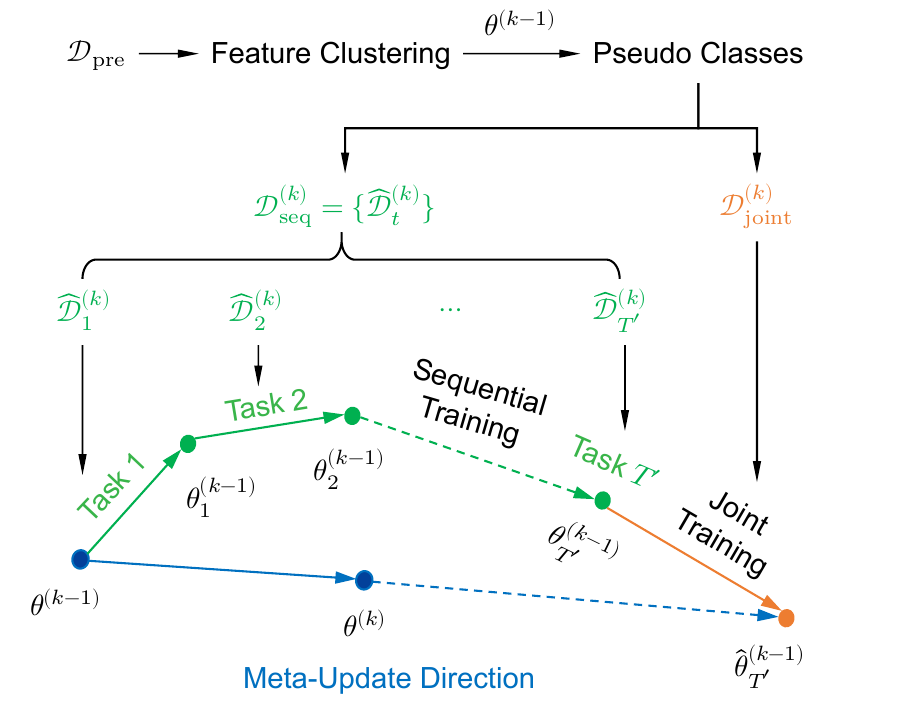}
    \caption{\textbf{Bi-level meta-refinement in MetaPrep.} At each meta-epoch, unlabeled upstream samples are clustered into pseudo classes and arranged as a pseudo continual sequence together with a held-out joint evaluation set. Sequential adaptation precedes joint refinement, and their parameter displacement defines the meta-update toward an initialization optimized for continual adaptation.}
    \label{fig:metaprep}
\end{figure}

\myPara{Unsupervised Continual Sequence Construction}
Let $\mathcal{D}_{\mathrm{pre}}=\{\mathbf{x}_i\}_{i=1}^{N}$ denote the unlabeled upstream dataset containing $N$ samples, and initialize the meta-refinement with $\boldsymbol{\theta}^{(0)}=\boldsymbol{\theta}_0$. At meta-epoch $k\in\{1,\ldots,K\}$, where $K$ denotes the total number of meta-epochs, we first extract the current representation of each sample as follows:
\begin{equation}
\mathbf{f}_i^{(k)}=f_{\boldsymbol{\theta}^{(k-1)}}(\mathbf{x}_i), \qquad i=1,\ldots,N,
\label{eq:metaprep_feature}
\end{equation}
where $\mathbf{f}_i^{(k)}\in\mathbb{R}^{d}$ is the $d$-dimensional feature produced by the encoder after the $(k-1)$-th meta-update. Since semantic annotations of upstream data may be unavailable, we recover the intrinsic structure already encoded in these features through $M$-way K-means clustering over the current feature set:
\begin{equation}
\{\mathbf{c}_m^{(k)}\}_{m=1}^{M}
=
\arg\min_{\{\mathbf{c}_m\}_{m=1}^{M}}
\sum_{i=1}^{N}
\min_{m\in\{1,\ldots,M\}}
\left\|
\mathbf{f}_i^{(k)}-\mathbf{c}_m
\right\|_2^2,
\label{eq:metaprep_cluster}
\end{equation}
where $M$ is the number of pseudo-classes and $\mathbf{c}_m^{(k)}\in\mathbb{R}^{d}$ denotes the centroid of the $m$-th cluster at meta-epoch $k$. Each sample is then assigned the pseudo-label
\begin{equation}
\tilde{y}_i^{(k)}
=
\arg\min_{m\in\{1,\ldots,M\}}
\left\|
\mathbf{f}_i^{(k)}-\mathbf{c}_m^{(k)}
\right\|_2^2.
\label{eq:metaprep_pseudolabel}
\end{equation}
The pseudo-label $\tilde{y}_i^{(k)}$ therefore reflects the latent semantic neighborhood of $\mathbf{x}_i$ in the current pretrained representation space rather than an external annotation. We reserve a held-out joint set $\mathcal{D}_{\mathrm{joint}}^{(k)}$ containing samples from the discovered pseudo-classes for post-sequence evaluation, while the remaining pseudo-labeled samples are organized into a sequential meta-training stream. Specifically, let $\{\mathcal{C}_{t}^{(k)}\}_{t=1}^{T'}$ denote $T'$ disjoint subsets of pseudo-class indices, where $\mathcal{C}_{t}^{(k)}\cap\mathcal{C}_{s}^{(k)}=\varnothing$ for $t\neq s$. The $t$-th pseudo continual task is defined by the following sample set for sequential adaptation:
\begin{equation}
\hat{\mathcal{D}}_{t}^{(k)}
=
\left\{
(\mathbf{x}_i,\tilde{y}_i^{(k)})
\,\middle|\,
\tilde{y}_i^{(k)}\in\mathcal{C}_{t}^{(k)}
\right\},
\qquad
t=1,\ldots,T',
\label{eq:metaprep_sequence}
\end{equation}
where $T'$ denotes the number of pseudo tasks within each simulated continual sequence. This construction converts static upstream observations into sequentially emerging pseudo concepts, allowing post-training to expose the encoder to the representation interference that will later arise in GCL. \cref{eq:metaprep_cluster,eq:metaprep_pseudolabel} are recomputed at every meta-epoch, so the simulated continual experiences evolve together with the representation rather than being fixed by predetermined semantic labels.

\myPara{Bi-level Meta-Refinement}
Given the pseudo continual sequence $\{\hat{\mathcal{D}}_{t}^{(k)}\}_{t=1}^{T'}$, we optimize the encoder according to its behavior after sequential adaptation rather than its performance on shuffled upstream samples. At meta-epoch $k$, we initialize the inner-loop encoder as $\boldsymbol{\theta}_{0}^{(k)}=\boldsymbol{\theta}^{(k-1)}$ and introduce an auxiliary prediction head $g_{\boldsymbol{\psi}}$ with initial parameters $\boldsymbol{\psi}_{0}^{(k)}$. For pseudo task $t$, its classification objective is defined as follows:
\begin{equation}
\mathcal{L}_{t}^{(k)}(\boldsymbol{\theta},\boldsymbol{\psi})
=
\mathbb{E}_{(\mathbf{x},\tilde{y})\sim\hat{\mathcal{D}}_{t}^{(k)}}
\left[
\mathcal{L}_{\mathrm{CE}}
\left(
g_{\boldsymbol{\psi}}
\left(
f_{\boldsymbol{\theta}}(\mathbf{x})
\right),
\tilde{y}
\right)
\right],
\label{eq:metaprep_task_loss}
\end{equation}
where $\mathcal{L}_{\mathrm{CE}}$ denotes the cross-entropy loss and $\tilde{y}$ is the pseudo-label obtained from \cref{eq:metaprep_pseudolabel}. The encoder and auxiliary head are then updated sequentially over the pseudo tasks:
\begin{align}
\boldsymbol{\theta}_{t}^{(k)}
&=
\boldsymbol{\theta}_{t-1}^{(k)}
-
\eta_{\theta}
\nabla_{\boldsymbol{\theta}}
\mathcal{L}_{t}^{(k)}
\left(
\boldsymbol{\theta}_{t-1}^{(k)},
\boldsymbol{\psi}_{t-1}^{(k)}
\right), \\
\boldsymbol{\psi}_{t}^{(k)}
&=
\boldsymbol{\psi}_{t-1}^{(k)}
-
\eta_{\psi}
\nabla_{\boldsymbol{\psi}}
\mathcal{L}_{t}^{(k)}
\left(
\boldsymbol{\theta}_{t-1}^{(k)},
\boldsymbol{\psi}_{t-1}^{(k)}
\right),
\label{eq:metaprep_inner}
\end{align}
where $\eta_{\theta}$ and $\eta_{\psi}$ denote the inner-loop learning rates for the encoder and prediction head, respectively. Unlike conventional post-training on i.i.d.\ batches, these ordered updates deliberately perturb the representation through successive pseudo concepts, thereby simulating the acquisition--interference dynamics encountered during downstream GCL. After processing all $T'$ pseudo tasks, we evaluate whether the sequentially adapted encoder $\boldsymbol{\theta}_{T'}^{(k)}$ still preserves knowledge across the full pseudo-semantic space using the held-out joint set for global evaluation over all discovered concepts:
\begin{equation}
\mathcal{L}_{\mathrm{joint}}^{(k)}
=
\mathbb{E}_{(\mathbf{x},\tilde{y})\sim\mathcal{D}_{\mathrm{joint}}^{(k)}}
\left[
\mathcal{L}_{\mathrm{CE}}
\left(
g_{\boldsymbol{\psi}_{T'}^{(k)}}
\left(
f_{\boldsymbol{\theta}_{T'}^{(k)}}(\mathbf{x})
\right),
\tilde{y}
\right)
\right].
\label{eq:metaprep_joint_loss}
\end{equation}
Here, $\mathcal{D}_{\mathrm{joint}}^{(k)}$ jointly covers the discovered pseudo concepts rather than only the latest task, so minimizing $\mathcal{L}_{\mathrm{joint}}^{(k)}$ explicitly evaluates the representation after sequential perturbation from a global perspective. We then perform one post-sequence refinement step on this joint objective over the shared pseudo-semantic space before the outer meta-update:
\begin{equation}
\hat{\boldsymbol{\theta}}_{T'}^{(k)}
=
\boldsymbol{\theta}_{T'}^{(k)}
-
\eta_{\theta}
\nabla_{\boldsymbol{\theta}}
\mathcal{L}_{\mathrm{joint}}^{(k)},
\label{eq:metaprep_outer}
\end{equation}
where $\hat{\boldsymbol{\theta}}_{T'}^{(k)}$ denotes the parameters that retain good joint performance after the simulated continual trajectory. Following the first-order meta-learning principle of Reptile~\cite{nichol2018reptile}, we then update the initialization toward this post-adaptation solution:
\begin{equation}
\boldsymbol{\theta}^{(k)}
=
\boldsymbol{\theta}^{(k-1)}
+
\eta_{\mathrm{meta}}
\left(
\hat{\boldsymbol{\theta}}_{T'}^{(k)}
-
\boldsymbol{\theta}^{(k-1)}
\right),
\label{eq:metaprep_meta}
\end{equation}
where $\eta_{\mathrm{meta}}$ is the meta learning rate controlling the magnitude of the outer update. This update does not merely favor parameters that perform well before adaptation; instead, it moves the pretrained initialization toward regions that remain effective after a sequence of continual updates. Repeating the above procedure for $K$ meta-epochs yields $\boldsymbol{\theta}^{\ast}=\boldsymbol{\theta}^{(K)}$,
where $\boldsymbol{\theta}^{\ast}$ denotes the refined encoder parameters used to initialize downstream GCL. In this way, the sequential inner loop supplies controlled continual perturbations, while the joint outer evaluation provides the meta-signal that favors post-adaptation generalization, jointly turning static post-training into continual-oriented representation refinement with improved adaptability for downstream continual learning.

\myPara{Stable Representation Prior}
MetaPrep produces not only the refined initialization $\boldsymbol{\theta}^{\ast}$ but also a stable summary of its global representation structure, which later serves as the interface to downstream reconciliation. We first extract the refined upstream features $f_{\boldsymbol{\theta}^{\ast}}(\mathbf{x}_i)$ and repeat the clustering procedure in \cref{eq:metaprep_cluster,eq:metaprep_pseudolabel} to obtain final pseudo-class assignments $\tilde{y}_i$. For pseudo-class $m\in\{1,\ldots,M\}$, let $\mathcal{I}_m=\{i\mid\tilde{y}_i=m\}$ denote the indices of its assigned samples and $N_m=|\mathcal{I}_m|$ the corresponding number of samples. We summarize each pseudo-class using its prototype
\begin{equation}
\boldsymbol{\mu}_{m}
=
\frac{1}{N_m}
\sum_{i\in\mathcal{I}_m}
f_{\boldsymbol{\theta}^{\ast}}(\mathbf{x}_i),
\qquad
m=1,\ldots,M,
\label{eq:metaprep_proto}
\end{equation}
where $\boldsymbol{\mu}_{m}\in\mathbb{R}^{d}$ represents the semantic center of pseudo-class $m$. Using prototypes rather than individual samples suppresses within-cluster variation and allows the prior to emphasize the global organization among the discovered semantic groups. We then compute the mean prototype $
\bar{\boldsymbol{\mu}}
=
\frac{1}{M}
\sum_{m=1}^{M}
\boldsymbol{\mu}_{m}$,
and characterize the between-prototype second-order geometry as
\begin{equation}
\boldsymbol{\Sigma}_{\mathrm{pre}}
=
\frac{1}{M-1}
\sum_{m=1}^{M}
\left(
\boldsymbol{\mu}_{m}-\bar{\boldsymbol{\mu}}
\right)
\left(
\boldsymbol{\mu}_{m}-\bar{\boldsymbol{\mu}}
\right)^{\top}
\in\mathbb{R}^{d\times d}.
\label{eq:metaprep_cov}
\end{equation}
Here, $\boldsymbol{\Sigma}_{\mathrm{pre}}$ captures the global second-order organization of the refined pseudo-semantic centers and is computed once before downstream adaptation. It therefore differs from statistics estimated from sparse and temporally mixed online batches: the former provides a fixed upstream reference, whereas the latter reflects transient downstream observations. MetaPrep can consequently be summarized by the interface
\begin{equation}
\operatorname{MetaPrep}
\left(
\boldsymbol{\theta}_{0},
\mathcal{D}_{\mathrm{pre}}
\right)
\longrightarrow
\left(
\boldsymbol{\theta}^{\ast},
\boldsymbol{\Sigma}_{\mathrm{pre}}
\right),
\label{eq:metaprep_interface}
\end{equation}
where $\boldsymbol{\theta}^{\ast}$ initializes the downstream continual learner and $\boldsymbol{\Sigma}_{\mathrm{pre}}$ provides the stable reference geometry consumed by StreamAlign during online adaptation, as detailed in \cref{sec.streamalign}.

\subsection{StreamAlign: Stream-Robust Representation Reconciliation}
\label{sec.streamalign}

The empirical observations in \cref{subsec:representation_alignment} reveal a downstream alignment gap: although output alignment can improve PTM-based GCL, its effectiveness becomes unreliable when the required statistics are estimated from sparse and blurry online streams. StreamAlign addresses this issue by using the stable upstream geometry produced by MetaPrep as an external reference and reconciling each transient online representation against this structure during continual adaptation. Specifically, StreamAlign consists of three steps: (1) reconstructing transient online features toward the stable reference geometry, (2) enforcing semantic consistency between the original and reconstructed representations, and (3) jointly optimizing prediction and reconciliation objectives throughout the online stream. As illustrated in \cref{fig:streamalign}, directly aligning transient representations toward the reference geometry can improve distributional balance but may simultaneously weaken semantic separability. StreamAlign therefore reconciles the original plastic representation with its reference-aligned counterpart, aiming to preserve both geometric stability and discriminative structure. Given the refined encoder $\boldsymbol{\theta}^{\ast}$ and representation prior $\boldsymbol{\Sigma}_{\mathrm{pre}}$ produced by MetaPrep, StreamAlign transforms each evolving online representation into a reference-aligned yet discriminative feature for downstream prediction.

\begin{figure}
    \centering
    \includegraphics[width=0.8\linewidth,clip,trim=150 20 150 20]{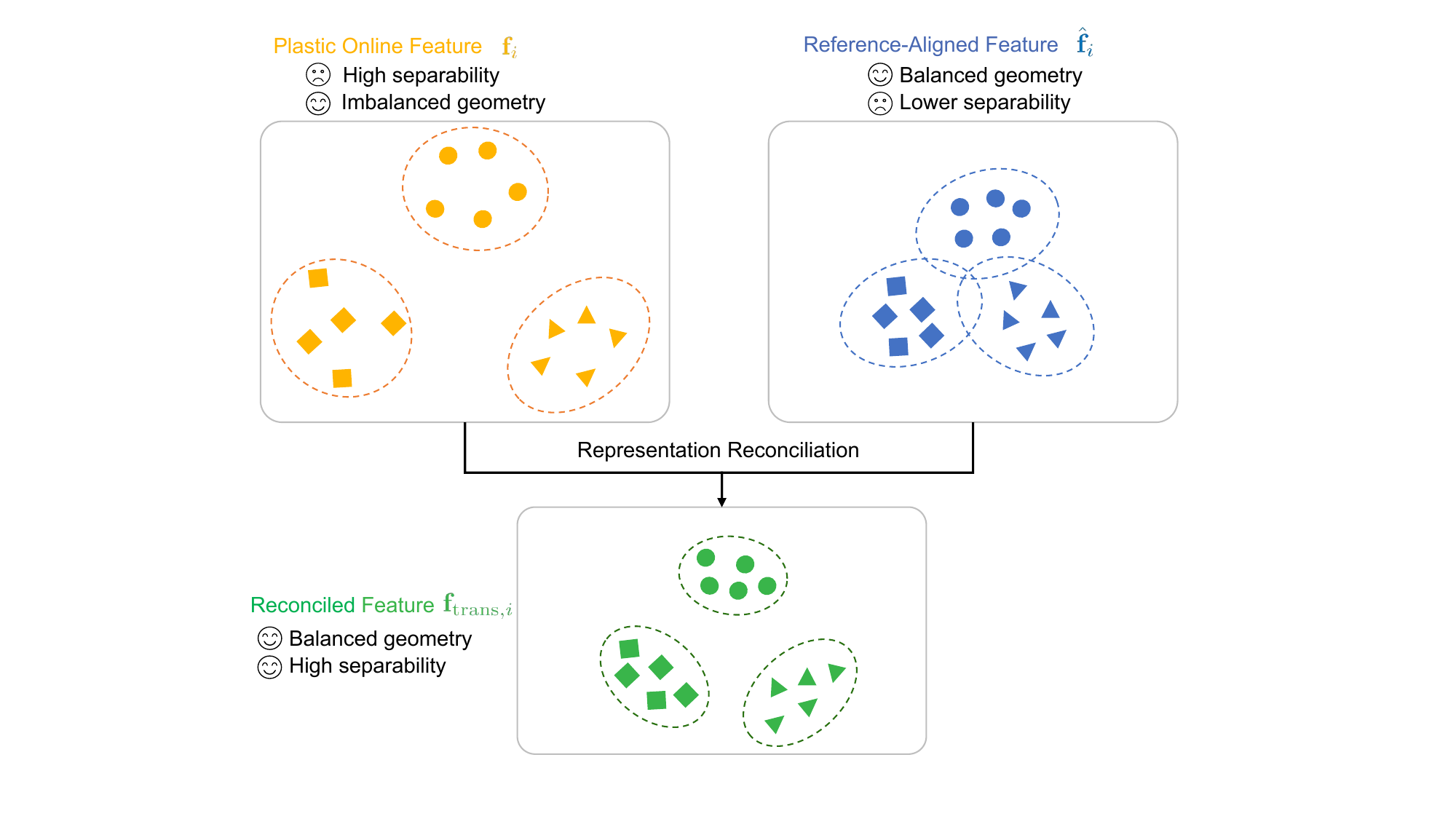}
    \caption{\textbf{Illustration of representation reconciliation in StreamAlign.} The transient pre-aligned representation $\mathbf{f}_i$ may preserve strong class separability but exhibit an imbalanced geometry, whereas direct geometric alignment produces a more balanced representation $\hat{\mathbf{f}}_i$ at the potential cost of reduced semantic separability. StreamAlign reconciles the two representations into $\mathbf{f}_{\mathrm{trans},i}$, preserving both geometric balance and discriminative structure for continual prediction.}
    \label{fig:streamalign}
\end{figure}

\myPara{Transient-to-Reference Reconstruction}
At learning step $t$, let $\mathcal{B}_{t}=\{(\mathbf{x}_{i},y_i)\}_{i=1}^{N_t}$ denote the current labeled mini-batch containing $N_t$ observations. Starting from the MetaPrep initialization $\boldsymbol{\theta}^{\ast}$, the downstream learner progressively updates the encoder to $\boldsymbol{\theta}_{t}$ and extracts the online feature for each incoming sample at every step as follows:
\begin{equation}
\mathbf{f}_{i}^{t}
=
f_{\boldsymbol{\theta}_{t}}(\mathbf{x}_{i})
\in\mathbb{R}^{d},
\qquad
i=1,\ldots,N_t,
\label{eq:stream_feature}
\end{equation}
where $\mathbf{f}_{i}^{t}$ reflects the representation after adaptation to the observations available up to step $t$. Because $\mathcal{B}_{t}$ is typically small and may contain temporally mixed concepts, its local feature distribution can substantially deviate from the globally organized representation space established before deployment. We characterize this transient geometry using the following batch mean for the current observations:
$\bar{\mathbf{f}}_{t}
=
\frac{1}{N_t}
\sum_{i=1}^{N_t}
\mathbf{f}_{i}^{t}$,
and the corresponding covariance
\begin{equation}
\boldsymbol{\Sigma}_{\mathrm{cur}}^{t}
=
\frac{1}{N_t-1}
\sum_{i=1}^{N_t}
\left(
\mathbf{f}_{i}^{t}-\bar{\mathbf{f}}_{t}
\right)
\left(
\mathbf{f}_{i}^{t}-\bar{\mathbf{f}}_{t}
\right)^{\top}
\in\mathbb{R}^{d\times d}.
\label{eq:stream_cov}
\end{equation}
Here, $\boldsymbol{\Sigma}_{\mathrm{cur}}^{t}$ summarizes the second-order geometry induced by the current online observations, whereas the fixed $\boldsymbol{\Sigma}_{\mathrm{pre}}$ obtained in \cref{eq:metaprep_cov} represents the stable reference geometry distilled from the refined upstream representation. To reduce the influence of transient online bias, we construct a linear transformation that maps the current covariance structure toward the reference structure. Let
\begin{equation}
\boldsymbol{\Sigma}_{\mathrm{cur}}^{t}
=
\boldsymbol{L}_{\mathrm{cur}}^{t}
\left(
\boldsymbol{L}_{\mathrm{cur}}^{t}
\right)^{\top},
\qquad
\boldsymbol{\Sigma}_{\mathrm{pre}}
=
\boldsymbol{L}_{\mathrm{pre}}
\boldsymbol{L}_{\mathrm{pre}}^{\top}
\label{eq:stream_cholesky}
\end{equation}
denote the Cholesky factorizations of the current and reference covariance matrices, where $\boldsymbol{L}_{\mathrm{cur}}^{t}$ and $\boldsymbol{L}_{\mathrm{pre}}$ are lower-triangular factors. The corresponding covariance transport is defined as
\begin{equation}
\boldsymbol{A}_{t}
=
\left(
\boldsymbol{L}_{\mathrm{cur}}^{t}
\right)^{-1}
\boldsymbol{L}_{\mathrm{pre}},
\label{eq:stream_transform}
\end{equation}
so that applying $\boldsymbol{A}_{t}$ moves the second-order structure of the current representation toward the upstream reference. We reconstruct each online feature for the current batch using the following linear transformation:
$
\hat{\mathbf{f}}_{i}^{t}
=
\mathbf{f}_{i}^{t}
\boldsymbol{A}_{t}
$,
where $\hat{\mathbf{f}}_{i}^{t}$ denotes the reference-aligned representation. Directly replacing $\mathbf{f}_{i}^{t}$ with $\hat{\mathbf{f}}_{i}^{t}$, however, may suppress useful changes learned from the current stream. We therefore retain both the plastic online representation and the stable reconstructed representation through
\begin{equation}
\mathbf{f}_{\mathrm{trans},i}^{t}
=
(1-\alpha)\mathbf{f}_{i}^{t}
+
\alpha\hat{\mathbf{f}}_{i}^{t},
\qquad
\alpha\in[0,1],
\label{eq:stream_interp}
\end{equation}
where $\alpha$ controls the strength of reference alignment. This preserves stream-specific adaptation through $\mathbf{f}_{i}^{t}$ while anchoring the resulting feature to the stable geometry encoded by $\hat{\mathbf{f}}_{i}^{t}$.

\myPara{Semantic Representation Reconciliation}
Covariance reconstruction constrains the global second-order geometry of online features, but geometric consistency alone does not guarantee that the resulting representation remains semantically discriminative. In particular, different classes may undergo similar geometric transformations, causing class-level relationships to become distorted even when the global covariance structure is moved closer to the reference geometry. We therefore treat the original and reconstructed representations as two complementary views of the same online sample: $\mathbf{f}_{i}^{t}$ retains plastic information acquired from the evolving stream, whereas $\hat{\mathbf{f}}_{i}^{t}$ reflects the stable upstream geometry. Before measuring semantic agreement, we normalize both views at each learning step and use the resulting unit features in the batch-level contrastive objective:
\begin{equation}
\mathbf{z}_{i}^{(1)}
=
\frac{\mathbf{f}_{i}^{t}}
{\|\mathbf{f}_{i}^{t}\|_2},
\qquad
\mathbf{z}_{i}^{(2)}
=
\frac{\hat{\mathbf{f}}_{i}^{t}}
{\|\hat{\mathbf{f}}_{i}^{t}\|_2},
\label{eq:stream_views}
\end{equation}
and collect the resulting $2N_t$ representations into the contrastive index set $\mathcal{V}_{t}$. For an anchor representation indexed by $a\in\mathcal{V}_{t}$, let
$
\mathcal{P}_{t}(a)
=
\left\{
p\in\mathcal{V}_{t}
\setminus\{a\}
\mid
y_p=y_a
\right\}
$
denote the set of representations that share the same semantic label as the anchor, where $y_a$ denotes the class label inherited from the corresponding online sample. For each anchor and its valid positive samples in the current batch, we define the pairwise contrastive term used for semantic reconciliation as follows:
\begin{equation}
\ell_{a,p}^{t}
=
-\log
\frac{\exp(\mathbf{z}_{a}^{\top}\mathbf{z}_{p}/\tau)}
{\sum_{q\in\mathcal{V}_{t}\setminus\{a\}}
\exp(\mathbf{z}_{a}^{\top}\mathbf{z}_{q}/\tau)},
\label{eq:stream_pair_con}
\end{equation}
where $\tau>0$ denotes the temperature parameter. The semantic reconciliation loss is then defined over all valid anchors in the current batch and their positive pairs as follows:
\begin{equation}
\mathcal{L}_{\mathrm{con}}^{t}
=
\frac{1}{|\mathcal{I}_{t}|}
\sum_{a\in\mathcal{I}_{t}}
\frac{1}{|\mathcal{P}_{t}(a)|}
\sum_{p\in\mathcal{P}_{t}(a)}
\ell_{a,p}^{t},
\label{eq:stream_con}
\end{equation}
where $\mathcal{I}_{t}=\{a\in\mathcal{V}_{t}\mid |\mathcal{P}_{t}(a)|>0\}$ denotes the set of valid anchors with at least one positive representation. This objective encourages representations from the same class to remain consistent across the plastic and reference-aligned spaces while explicitly separating representations from different classes. As a result, StreamAlign does not merely restore global covariance structure but reconciles the two representation states at the semantic level, preventing geometric stabilization from sacrificing discriminability in the continual learning process.

\begin{algorithm}[t]
\caption{The procedure of MetaPrep.}
\label{alg.metaprep}
\begin{algorithmic}[1]
\STATE \textbf{Input:} Unlabeled upstream data $\mathcal{D}_{\mathrm{pre}}$, pretrained encoder $\boldsymbol{\theta}^{(0)}$
\STATE \textbf{Hyperparameters:} Meta-epochs $K$, pseudo-classes $M$, pseudo-tasks $T'$, meta step size $\eta_{\mathrm{meta}}$
\FOR{$k=1$ to $K$}
    \STATE Extract features $f_{\boldsymbol{\theta}^{(k-1)}}(\mathcal{D}_{\mathrm{pre}})$ and cluster them into $M$ pseudo-classes
    \STATE Construct pseudo continual sequence $\{\hat{\mathcal{D}}_{t}^{(k)}\}_{t=1}^{T'}$ and held-out joint set $\mathcal{D}_{\mathrm{joint}}^{(k)}$
    \STATE Initialize $\boldsymbol{\theta}_{0}^{(k)} \gets \boldsymbol{\theta}^{(k-1)}$ and auxiliary head $\boldsymbol{\psi}_{0}^{(k)}$
    \FOR{$t=1$ to $T'$}
        \STATE $(\boldsymbol{\theta}_{t}^{(k)},\boldsymbol{\psi}_{t}^{(k)})
        \gets
        \operatorname{Update}
        (\boldsymbol{\theta}_{t-1}^{(k)},\boldsymbol{\psi}_{t-1}^{(k)};
        \hat{\mathcal{D}}_{t}^{(k)})$
    \ENDFOR
    \STATE $\hat{\boldsymbol{\theta}}_{T'}^{(k)}
    \gets
    \operatorname{Update}_{\boldsymbol{\theta}}
    (\boldsymbol{\theta}_{T'}^{(k)},\boldsymbol{\psi}_{T'}^{(k)};
    \mathcal{D}_{\mathrm{joint}}^{(k)})$
    \STATE $\boldsymbol{\theta}^{(k)}
    \gets
    \boldsymbol{\theta}^{(k-1)}
    +
    \eta_{\mathrm{meta}}
    \big(
    \hat{\boldsymbol{\theta}}_{T'}^{(k)}
    -
    \boldsymbol{\theta}^{(k-1)}
    \big)$
\ENDFOR
\STATE Set GCL-oriented initialization $\boldsymbol{\theta}^{\ast} \gets \boldsymbol{\theta}^{(K)}$
\STATE Re-cluster $f_{\boldsymbol{\theta}^{\ast}}(\mathcal{D}_{\mathrm{pre}})$ and compute pseudo-class prototypes $\{\boldsymbol{\mu}_{m}\}_{m=1}^{M}$
\STATE Compute stable geometry prior $\boldsymbol{\Sigma}_{\mathrm{pre}}$ from $\{\boldsymbol{\mu}_{m}\}_{m=1}^{M}$
\STATE \textbf{Return:} $\boldsymbol{\theta}^{\ast}$, $\boldsymbol{\Sigma}_{\mathrm{pre}}$
\end{algorithmic}
\end{algorithm}

\myPara{Online Learning Objective}
The reconciled feature $\mathbf{f}_{\mathrm{trans},i}^{t}$ is used for downstream prediction through the current classifier $g_{\boldsymbol{\phi}_{t}}$. The primary online classification objective is
\begin{equation}
\mathcal{L}_{\mathrm{main}}^{t}
=
\frac{1}{N_t}
\sum_{i=1}^{N_t}
\mathcal{L}_{\mathrm{CE}}
\left(
g_{\boldsymbol{\phi}_{t}}
\left(
\mathbf{f}_{\mathrm{trans},i}^{t}
\right),
y_i
\right),
\label{eq:stream_main}
\end{equation}
where $\boldsymbol{\phi}_{t}$ denotes the downstream prediction parameters at step $t$. We jointly optimize classification and semantic reconciliation using the following online objective:
\begin{equation}
\mathcal{L}_{\mathrm{GCL}}^{t}
=
\mathcal{L}_{\mathrm{main}}^{t}
+
\lambda
\mathcal{L}_{\mathrm{con}}^{t},
\label{eq:stream_total}
\end{equation}
where $\lambda\geq0$ controls the contribution of semantic reconciliation. The downstream parameters are updated according to
\begin{equation}
(\boldsymbol{\theta}_{t+1},\boldsymbol{\phi}_{t+1})
=
\mathcal{A}
\left(
\boldsymbol{\theta}_{t},
\boldsymbol{\phi}_{t},
\mathcal{B}_{t};
\mathcal{L}_{\mathrm{GCL}}^{t}
\right),
\label{eq:stream_update}
\end{equation}
where $\mathcal{A}$ denotes the underlying GCL learner introduced in \cref{subsec:problem_formulation}. This formulation keeps StreamAlign agnostic to a particular downstream continual-learning algorithm: $\mathcal{L}_{\mathrm{main}}^{t}$ preserves responsiveness to the incoming stream, while $\mathcal{L}_{\mathrm{con}}^{t}$ constrains the evolving representation to remain semantically compatible with the stable upstream geometry. StreamAlign can therefore be summarized as
\begin{equation}
\operatorname{StreamAlign}
\left(
\mathbf{f}_{i}^{t},
\boldsymbol{\Sigma}_{\mathrm{cur}}^{t},
\boldsymbol{\Sigma}_{\mathrm{pre}}
\right)
\longrightarrow
\mathbf{f}_{\mathrm{trans},i}^{t},
\label{eq:stream_interface}
\end{equation}
completing the second stage of MePo++ by reconciling plastic online adaptation with a stable and discriminative representation structure throughout the evolving data stream.

\begin{algorithm}[t]
\caption{The procedure of StreamAlign.}
\label{alg.streamalign}
\begin{algorithmic}[1]
\STATE \textbf{Input:} GCL stream $\mathcal{D}_{\mathrm{GCL}}$, refined encoder $\boldsymbol{\theta}^{\ast}$, \par geometry prior $\boldsymbol{\Sigma}_{\mathrm{pre}}$
\STATE \textbf{Hyperparameters:} Reconciliation weight $\alpha$, contrastive weight $\lambda$, temperature $\tau$
\STATE Initialize online encoder $\boldsymbol{\theta} \gets \boldsymbol{\theta}^{\ast}$ and prediction head $\boldsymbol{\psi}$
\FOR{each incoming batch $\mathcal{B}$ in $\mathcal{D}_{\mathrm{GCL}}$}
    \STATE Extract online representations $\{\mathbf{f}_{i}\}$ using $f_{\boldsymbol{\theta}}$ and estimate $\boldsymbol{\Sigma}_{\mathrm{cur}}$
    \STATE Compute geometry transform $\boldsymbol{A}$ from $\boldsymbol{\Sigma}_{\mathrm{cur}}$ to $\boldsymbol{\Sigma}_{\mathrm{pre}}$
    \STATE Reconstruct reference-aligned features $\hat{\mathbf{f}}_{i} \gets \mathbf{f}_{i}\boldsymbol{A}$
    \STATE Reconcile representations:
    $\mathbf{f}_{\mathrm{trans},i}
    \gets
    (1-\alpha)\mathbf{f}_{i}
    +
    \alpha\hat{\mathbf{f}}_{i}$
    \STATE Compute prediction loss $\mathcal{L}_{\mathrm{main}}$ on $\{\mathbf{f}_{\mathrm{trans},i}\}$
    \STATE Compute semantic reconciliation loss $\mathcal{L}_{\mathrm{con}}$ between $\{\mathbf{f}_{i}\}$ and $\{\hat{\mathbf{f}}_{i}\}$ with temperature $\tau$
    \STATE Update $(\boldsymbol{\theta},\boldsymbol{\psi})$ using
    $\mathcal{L}_{\mathrm{GCL}}
    =
    \mathcal{L}_{\mathrm{main}}
    +
    \lambda\mathcal{L}_{\mathrm{con}}$
\ENDFOR
\STATE \textbf{Return:} Adapted encoder $\boldsymbol{\theta}$ and prediction head $\boldsymbol{\psi}$
\end{algorithmic}
\end{algorithm}

\subsection{Overall Learning Procedure}

The complete MePo++ framework integrates MetaPrep and StreamAlign into a two-stage learning procedure. As summarized in \cref{alg.metaprep}, MetaPrep first refines the pretrained encoder on unlabeled upstream data and produces both a GCL-oriented initialization $\boldsymbol{\theta}^{\ast}$ and a stable geometry prior $\boldsymbol{\Sigma}_{\mathrm{pre}}$. These outputs are then transferred to downstream continual adaptation, where StreamAlign, detailed in \cref{alg.streamalign}, continuously reconciles transient online representations with the pretrained geometry while optimizing the prediction objective. In this way, the two algorithms operate sequentially but complementarily: MetaPrep prepares the representation before deployment, whereas StreamAlign maintains geometric stability and semantic discriminability as the downstream stream evolves across all learning stages.

\begin{table*}[t]
\centering
\caption{
Overall performance comparison of different methods under the GCL setting
on 5-task CIFAR-100, ImageNet-R, and CUB-200.
All results are averaged over five runs with different task sequences and
reported as mean $\pm$ standard deviation.
}
\label{tab:main}

\setlength{\tabcolsep}{4.2pt}
\resizebox{0.75\textwidth}{!}{%
\begin{tabular}{
    c
    l
    *{6}{
        S[
            table-format=2.2,
            table-space-text-post={\std{11.55}}
        ]
    }
}
\toprule
\multirow{2.5}{*}{\textbf{PTM}}
&
\multirow{2.5}{*}{\textbf{Method}}
&
\multicolumn{2}{c}{\textbf{CIFAR-100}}
&
\multicolumn{2}{c}{\textbf{ImageNet-R}}
&
\multicolumn{2}{c}{\textbf{CUB-200}}
\\
\cmidrule(lr){3-4}
\cmidrule(lr){5-6}
\cmidrule(lr){7-8}
&
&
{$A_{\mathrm{AUC}}~(\uparrow)$}
&
{$A_{\mathrm{AUC}}~(\uparrow)$}
&
{$A_{\mathrm{Last}}~(\uparrow)$}
&
{$A_{\mathrm{AUC}}~(\uparrow)$}
&
{$A_{\mathrm{Last}}~(\uparrow)$}
&
{$A_{\mathrm{Last}}~(\uparrow)$}
\\
\midrule

\multirow{16}{*}{\shortstack{Sup-21K}}
& Seq FT
& 19.71\std{3.39}
& 10.42\std{4.92}
& 7.51\std{3.94}
& 2.29\std{0.85}
& 3.47\std{0.41}
& 1.49\std{0.42}
\\

& Linear Probe
& 49.69\std{6.09}
& 23.07\std{7.33}
& 29.24\std{1.26}
& 16.87\std{3.14}
& 28.96\std{2.46}
& 17.33\std{3.08}
\\

& Seq FT (SL~\cite{zhang2023slca})
& 64.90\std{7.18}
& 62.06\std{1.89}
& 47.20\std{1.47}
& 39.60\std{2.43}
& 56.16\std{4.32}
& 56.50\std{3.08}
\\

& MVP~\cite{moon2023online}
& 68.13\std{4.34}
& 60.56\std{2.57}
& 41.50\std{1.15}
& 34.14\std{3.95}
& 56.78\std{2.88}
& 50.25\std{3.53}

\\

& CODA-P~\cite{smith2023coda}
& 78.81\std{3.38}
& 80.30\std{1.58}
& 50.11\std{2.14}
& 46.17\std{2.00}
& 64.96\std{3.30}
& 59.28\std{3.14}
\\

\addlinespace[3pt]

& L2P~\cite{wang2022learning}
& 78.12\std{0.61}
& 77.73\std{1.09}
& 42.39\std{0.23}
& 38.16\std{1.37}
& 60.95\std{1.22}
& 56.31\std{2.53}
\\

& \quad w/ MePo 
&  83.63\std{0.61}
&  83.98\std{0.29}
&  48.01\std{1.29}
&  45.89\std{0.60}
&  64.92\std{1.47}
&  63.30\std{1.52}
\\

& \quad w/ MePo++
& \bfseries86.04\std{1.12} 
& \bfseries84.88\std{0.34} 
& \bfseries50.02\std{1.30}
& \bfseries48.10\std{0.51}
& \bfseries67.05\std{1.58}
& \bfseries66.34\std{1.33}
\\

\addlinespace[3pt]

& DualPrompt~\cite{wang2022dualprompt}
& 66.36\std{4.42}
& 58.09\std{4.40}
& 38.63\std{2.19}
& 30.71\std{0.82}
& 55.73\std{2.77}
& 47.08\std{4.94}
\\

& \quad w/ MePo 
&  71.37\std{4.07}
&  66.48\std{2.82}
&  44.65\std{2.09}
&  36.76\std{1.21}
&  58.36\std{2.59}
&  52.16\std{3.74}
\\

& \quad w/ MePo++
& \bfseries72.99\std{4.11}
& \bfseries72.00\std{2.85}
& \bfseries47.34\std{1.63}
& \bfseries40.77\std{1.84}
& \bfseries59.05\std{2.98}
& \bfseries53.87\std{3.45}
\\

\addlinespace[3pt]

& MISA~\cite{kang2025advancing}
& 80.35\std{2.39}
& 80.75\std{1.24}
& 51.52\std{2.09}
& 45.08\std{1.43}
& 65.40\std{3.01}
& 60.20\std{1.82}
\\

& \quad w/ MePo 
&  82.30\std{2.83}
&  83.99\std{1.35}
&  54.86\std{2.20}
&  49.18\std{1.38}
&  68.13\std{3.17}
&  64.75\std{1.00}
\\

& \quad w/ MePo++
& \bfseries82.73\std{0.84}
& \bfseries84.01\std{0.50}
& \bfseries55.18\std{1.72}
& \bfseries50.20\std{1.43}
& \bfseries69.52\std{3.00}
& \bfseries68.63\std{1.39}
\\

\midrule

\multirow{12}{*}{\shortstack{Sup-21/1K}}

& L2P~\cite{wang2022learning}
& 69.15\std{1.66}
& 68.57\std{1.38}
& 42.74\std{0.83}
& 39.22\std{2.14}
& 39.20\std{1.69}
& 46.76\std{1.87}
\\

& \quad w/ MePo 
&  78.75\std{1.18}
&  77.52\std{1.03}
&  62.71\std{1.09}
&  58.91\std{0.08}
&  48.36\std{1.88}
&  50.88\std{2.85}
\\

& \quad w/ MePo++
& \bfseries79.33\std{0.61}
& \bfseries82.12\std{0.87}
& \bfseries64.12\std{1.21}
& \bfseries61.30\std{1.38}
& \bfseries48.83\std{1.49}
& \bfseries54.56\std{1.05}
\\

\addlinespace[3pt]

& DualPrompt~\cite{wang2022dualprompt}
& 64.84\std{2.62}
& 54.00\std{3.36}
& 48.61\std{0.76}
& 43.85\std{1.04}
& 44.08\std{2.80}
& 36.27\std{4.52}
\\

& \quad w/ MePo 
&  67.18\std{4.48}
&  57.95\std{3.69}
&  54.75\std{1.66}
&  44.75\std{0.74}
&  47.06\std{3.19}
&  38.24\std{9.29}
\\

& \quad w/ MePo++
& \bfseries67.52\std{4.46}
& \bfseries59.77\std{3.73}
& \bfseries55.99\std{1.22}
& \bfseries46.89\std{0.82}
& \bfseries47.48\std{2.92}
& \bfseries39.58\std{9.47}
\\

\addlinespace[3pt]

& MISA~\cite{kang2025advancing}
& 62.91\std{7.96}
& 67.99\std{7.41}
& 50.87\std{1.69}
& 47.75\std{2.87}
& 42.76\std{2.33}
& 44.05\std{1.94}
\\

& \quad w/ MePo 
&  78.73\std{2.43}
&  78.21\std{1.96}
&  64.23\std{1.30}
&  58.20\std{0.51}
&  55.31\std{4.52}
&  56.58\std{2.33}
\\

& \quad w/ MePo++
& \bfseries78.94\std{2.65}
& \bfseries78.35\std{2.06}
& \bfseries65.27\std{0.48}
& \bfseries60.10\std{0.77}
& \bfseries57.61\std{4.67}
& \bfseries58.85\std{5.80}
\\

\midrule

\multirow{12}{*}{\shortstack{DINOv2}}

& L2P~\cite{wang2022learning}
& 70.31\std{1.09}
& 73.61\std{0.93}
& 56.51\std{1.32}
& 53.94\std{1.83}
& 27.36\std{1.64}
& 46.34\std{2.97}
\\

& \quad w/ MePo 
& 73.53\std{0.75}
& 77.34\std{0.90}
& 66.25\std{0.64}
& 60.85\std{2.02}
& 31.86\std{1.29}
& 48.87\std{1.82}
\\

& \quad w/ MePo++
& \bfseries76.61\std{0.70}
& \bfseries79.62\std{0.87}
& \bfseries69.95\std{0.55}
& \bfseries66.44\std{1.76}
& \bfseries36.86\std{1.40}
& \bfseries51.73\std{1.90}
\\

\addlinespace[3pt]

& DualPrompt~\cite{wang2022dualprompt}
& 64.28\std{4.26}
& 54.86\std{2.64}
& 57.28\std{1.55}
& 46.84\std{2.53}
& 48.87\std{2.77}
& 42.86\std{5.31}
\\

& \quad w/ MePo 
& 65.38\std{4.11}
& 55.51\std{3.19}
& 58.51\std{1.64}
& 47.30\std{2.28}
& 50.41\std{3.11}
& 43.34\std{6.10}
\\

& \quad w/ MePo++
& \bfseries66.92\std{4.54}
& \bfseries57.53\std{2.98}
& \bfseries61.60\std{1.84}
& \bfseries49.77\std{2.83}
& \bfseries51.41\std{2.93}
& \bfseries44.18\std{5.93}
\\

\addlinespace[3pt]

& MISA~\cite{kang2025advancing}
& 70.73\std{3.59}
& 68.53\std{2.13}
& 63.70\std{1.46}
& 57.07\std{1.90}
& 52.71\std{3.52}
& 49.78\std{2.18}
\\

& \quad w/ MePo 
& 71.51\std{3.20}
& 69.08\std{1.58}
& 68.29\std{1.15}
& 60.73\std{1.07}
& 54.29\std{3.57}
& 49.45\std{2.07}
\\

& \quad w/ MePo++
& \bfseries74.78\std{3.69}
& \bfseries71.74\std{2.15}
& \bfseries70.88\std{0.99}
& \bfseries64.26\std{0.98}
& \bfseries57.44\std{3.10}
& \bfseries52.04\std{2.11}
\\

\bottomrule
\end{tabular}%
}

\end{table*}

\begin{table*}[t]
\centering
\caption{
Overall performance comparison of different methods under the few-shot GCL setting (\textbf{20\%} train dataset)
on 5-task CIFAR-100, ImageNet-R, and CUB-200.
All results are averaged over five runs with different task sequences and
reported as mean $\pm$ standard deviation for each configuration.
}
\label{tab:fewshot}

\setlength{\tabcolsep}{4.2pt}
\renewcommand{\arraystretch}{1.08}

\resizebox{0.75\textwidth}{!}{%
\begin{tabular}{
    c
    l
    *{6}{
        S[
            table-format=2.2,
            table-space-text-post={\std{11.55}}
        ]
    }
}
\toprule
\multirow{2.5}{*}{\textbf{PTM}}
&
\multirow{2.5}{*}{\textbf{Method}}
&
\multicolumn{2}{c}{\textbf{CIFAR-100}}
&
\multicolumn{2}{c}{\textbf{ImageNet-R}}
&
\multicolumn{2}{c}{\textbf{CUB-200}}
\\
\cmidrule(lr){3-4}
\cmidrule(lr){5-6}
\cmidrule(lr){7-8}
&
&
{$A_{\mathrm{AUC}}~(\uparrow)$}
&
{$A_{\mathrm{Last}}~(\uparrow)$}
&
{$A_{\mathrm{AUC}}~(\uparrow)$}
&
{$A_{\mathrm{Last}}~(\uparrow)$}
&
{$A_{\mathrm{AUC}}~(\uparrow)$}
&
{$A_{\mathrm{Last}}~(\uparrow)$}
\\
\midrule

\multirow{9}{*}{\shortstack{Sup-21K}}

& L2P~\cite{wang2022learning}
& 49.63\std{2.65}
& 52.09\std{2.49}
& 14.37\std{0.83}
& 12.82\std{0.76}
& 46.77\std{1.38}
& 44.89\std{1.35}
\\

& \quad w/ MePo
& 54.54\std{2.09}
& 58.94\std{0.94}
& 16.56\std{0.85}
& 15.06\std{1.08}
& 47.63\std{1.48}
& 47.74\std{1.70}
\\

& \quad w/ MePo++
& \bfseries58.16\std{2.23}
& \bfseries63.74\std{1.46}
& \bfseries21.72\std{0.69}
& \bfseries19.93\std{0.92}
& \bfseries50.56\std{1.37}
& \bfseries49.41\std{1.58}

\\

\addlinespace[3pt]

& DualPrompt~\cite{wang2022dualprompt}
& 52.13\std{3.42}
& 53.57\std{2.68}
& 23.35\std{0.93}
& 20.22\std{2.07}
& 48.86\std{3.69}
& 44.36\std{2.28}
\\

& \quad w/ MePo
& 54.02\std{4.76}
& 59.61\std{3.51}
& 22.31\std{1.26}
& 19.28\std{1.62}
& 48.82\std{2.75}
& 46.64\std{1.37}
\\

& \quad w/ MePo++
& \bfseries61.27\std{3.77}
& \bfseries69.59\std{2.89}
& \bfseries27.83\std{3.79}
& \bfseries24.99\std{2.29}
& \bfseries49.13\std{3.53}
& \bfseries50.25\std{2.12}

\\

\addlinespace[3pt]

& MISA~\cite{kang2025advancing}
& 61.32\std{2.40}
& 63.74\std{2.34}
& 26.03\std{1.79}
& 23.97\std{1.29}
& \bfseries55.77\std{2.55}
& 49.92\std{2.45}
\\

& \quad w/ MePo
& 59.11\std{2.50}
& 64.07\std{2.04}
& 24.77\std{1.85}
& 23.32\std{0.85}
& 53.99\std{2.04}
& 50.42\std{1.62}
\\

& \quad w/ MePo++
& \bfseries61.74\std{3.28}
& \bfseries65.23\std{1.88}
& \bfseries28.58\std{2.16}
& \bfseries26.68\std{1.03}
& 55.65\std{2.97}
& \bfseries52.72\std{1.10}

\\

\midrule

\multirow{9}{*}{\shortstack{Sup-21/1K}}
& L2P~\cite{wang2022learning}
& 33.70\std{1.58}
& 39.19\std{4.17}
& 13.06\std{1.34}
& 15.71\std{2.52}
& 15.95\std{1.34}
& 17.62\std{2.99}
\\

& \quad w/ MePo
& 48.42\std{3.86}
& 52.09\std{3.66}
& 32.51\std{1.66}
& 32.85\std{1.22}
& 23.42\std{1.51}
& 21.77\std{1.40}
\\

& \quad w/ MePo++
& \bfseries50.01\std{2.90}
& \bfseries53.84\std{2.84}
& \bfseries33.81\std{1.22}
& \bfseries34.06\std{0.92}
& \bfseries24.41\std{1.68}
& \bfseries22.70\std{1.48}
\\

\addlinespace[3pt]

& DualPrompt~\cite{wang2022dualprompt}
& 43.54\std{3.30}
& 47.29\std{3.70}
& 14.72\std{7.15}
& 8.24\std{3.17}
& 28.83\std{5.10}
& 26.51\std{4.04}
\\

& \quad w/ MePo
& 54.59\std{3.93}
& 55.08\std{2.39}
& 30.05\std{8.37}
& 27.79\std{3.06}
& 33.28\std{4.59}
& 30.86\std{2.01}
\\

& \quad w/ MePo++
& \bfseries56.03\std{3.35}
& \bfseries58.82\std{2.59}
& \bfseries39.62\std{5.08}
& \bfseries30.87\std{2.80}
& \bfseries34.08\std{6.05}
& \bfseries31.20\std{2.16}
\\

\addlinespace[3pt]

& MISA~\cite{kang2025advancing}
& 44.68\std{3.46}
& 45.34\std{2.14}
& 27.95\std{1.76}
& 24.20\std{0.65}
& 30.21\std{5.96}
& 21.01\std{2.05}
\\

& \quad w/ MePo
& 59.31\std{3.05}
& 58.19\std{1.70}
& 43.19\std{1.91}
& 40.19\std{1.35}
& 37.93\std{6.70}
& 28.33\std{2.52}
\\

& \quad w/ MePo++
& \bfseries60.24\std{5.04}
& \bfseries60.00\std{2.41}
& \bfseries45.64\std{1.58}
& \bfseries43.07\std{1.15}
& \bfseries40.40\std{6.20}
& \bfseries31.52\std{3.12}
\\

\midrule

\multirow{9}{*}{\shortstack{DINOv2}}
& L2P~\cite{wang2022learning}
& 23.70\std{2.45}
& 32.24\std{4.16}
& 11.18\std{1.07}
& 13.48\std{1.23}
& 5.95\std{0.14}
& 6.40\std{1.01}
\\

& \quad w/ MePo 
& 28.29\std{2.22}
& 39.83\std{4.11}
& 16.09\std{1.22}
& 18.07\std{1.13}
& 7.33\std{0.91}
& 7.19\std{0.37}
\\

& \quad w/ MePo++
& \bfseries31.67\std{2.46}
& \bfseries43.24\std{3.37}
& \bfseries21.61\std{3.07}
& \bfseries24.38\std{0.83}
& \bfseries8.82\std{0.90}
& \bfseries9.26\std{0.81}
\\

\addlinespace[3pt]

& DualPrompt~\cite{wang2022dualprompt}
& 52.68\std{1.45}
& 51.96\std{1.62}
& 20.26\std{13.59}
& 7.33\std{9.30}
& 13.15\std{6.75}
& 1.62\std{2.16}
\\

& \quad w/ MePo 
& 54.23\std{1.74}
& 54.50\std{1.51}
& 25.54\std{16.10}
& 9.39\std{18.54}
& 13.26\std{9.82}
& 5.90\std{9.08}
\\

& \quad w/ MePo++
& \bfseries55.61\std{1.72}
& \bfseries55.60\std{1.00}
& \bfseries29.09\std{10.83}
& \bfseries13.35\std{11.80}
& \bfseries17.19\std{13.52}
& \bfseries7.18\std{11.62}
\\

\addlinespace[3pt]

& MISA~\cite{kang2025advancing}
& 50.82\std{4.47}
& 48.80\std{5.10}
& 44.83\std{3.23}
& 35.80\std{0.47}
& 45.94\std{6.87}
& 37.35\std{1.88}
\\

& \quad w/ MePo 
& 55.77\std{3.95}
& 52.94\std{4.13}
& 49.71\std{3.18}
& 40.71\std{0.96}
& 50.84\std{7.89}
& 41.03\std{2.06}
\\

& \quad w/ MePo++
& \bfseries57.42\std{4.37}
& \bfseries54.98\std{2.70}
& \bfseries53.35\std{2.66}
& \bfseries46.06\std{1.28}
& \bfseries53.40\std{7.04}
& \bfseries44.27\std{2.24}
\\

\bottomrule
\end{tabular}%
}

\end{table*}

\begin{table}[t]
\centering
\caption{
Performance comparison of prompt-based learners with a CLIP backbone under the Si-Blurry GCL protocol on CIFAR-100 and ImageNet-R.
Results are averaged over five random seeds and reported as mean $\pm$ standard deviation.
Higher values are better for both metrics across all reported datasets.
}
\label{tab:vlm}

\setlength{\tabcolsep}{3.2pt}
\renewcommand{\arraystretch}{1.08}
\resizebox{\linewidth}{!}{%
\begin{tabular}{
    l
    *{4}{
        S[
            table-format=2.2,
            table-space-text-post={\std{2.60}}
        ]
    }
}
\toprule
\multirow{2.5}{*}{\textbf{Method}}
&
\multicolumn{2}{c}{\textbf{CIFAR-100}}
&
\multicolumn{2}{c}{\textbf{ImageNet-R}}
\\
\cmidrule(lr){2-3}
\cmidrule(lr){4-5}
&
{$A_{\mathrm{AUC}}~(\uparrow)$}
&
{$A_{\mathrm{Last}}~(\uparrow)$}
&
{$A_{\mathrm{AUC}}~(\uparrow)$}
&
{$A_{\mathrm{Last}}~(\uparrow)$}
\\
\midrule

L2P~\cite{wang2022learning}
& 76.34\std{1.58}
& 68.23\std{2.50}
& 76.51\std{1.12}
& 71.14\std{1.43}
\\

\quad w/ MePo++
& \bfseries 77.18\std{1.50}
& \bfseries 69.53\std{2.70}
& \bfseries 77.24\std{1.28}
& \bfseries 71.68\std{1.24}
\\

\addlinespace[3pt]
DualPrompt~\cite{wang2022dualprompt}
& 71.19\std{1.01}
& 65.83\std{0.19}
& 73.37\std{0.77}
& 66.56\std{0.11}
\\

\quad w/ MePo++
& \bfseries 74.27\std{1.07}
& \bfseries 68.72\std{0.17}
& \bfseries 76.29\std{0.79}
& \bfseries 69.90\std{0.18}
\\

\bottomrule
\end{tabular}%
}
\end{table}

\begin{table}[t]
    \centering
    \caption{
        Computational efficiency comparison on CUB-200 using the Sup-21K pretrained backbone.
    }
    \label{tab:complexity}

    \setlength{\tabcolsep}{3.2pt}
    \renewcommand{\arraystretch}{1.08}

    \resizebox{\linewidth}{!}{%
    \begin{tabular}{
        l
        r
        c
        c
        c
        r
        c
        c
        c
    }
        \toprule

        \multirow{2.5}{*}{\textbf{Setting}} & \multicolumn{4}{c}{\textbf{DualPrompt}} & \multicolumn{4}{c}{\textbf{MISA}} \\

        \cmidrule(lr){2-5}
        \cmidrule(lr){6-9}

        & {{+Param.}} & {+Ratio} & {{Time (s)}} & {{Accuracy}}
        & {{+Param.}} & {+Ratio} & {{Time (s)}} & {{Accuracy}} \\

        \midrule

        Baseline & {637k} & 0.74\% & 4.87 & 55.73 & {653k} & 0.76\% & 4.84 & 65.40 \\

        ~~w/ MePo & {1213k} & 1.41\% & 4.99 & 58.36 & {1229k} & 1.43\% & 5.01 & 68.13 \\

        ~~w/ MePo++ & {1213k} & 1.41\% & 5.09 & \bfseries 59.05 & {1229k} & 1.43\% & 5.05 & \bfseries 69.52 \\

        \bottomrule
    \end{tabular}%
    }
\end{table}

\section{Experiments}

We evaluate MePo++ across standard, few-shot, and CLIP-based continual-learning settings, followed by ablation and representation analyses to examine its effectiveness and underlying mechanisms under evolving data streams.

\subsection{Experimental Setting}
\label{sec:exp_setup}

\myPara{Datasets and Pretrained Backbones}
We evaluate MePo++ under the Si-Blurry GCL protocol on three representative image classification benchmarks with different scales and recognition characteristics: CIFAR-100~\cite{krizhevsky2009learning_cifar}, a 100-class general recognition dataset with low-resolution images; ImageNet-R~\cite{hendrycks2021many}, a 200-class large-scale benchmark containing diverse visual renditions; and CUB-200~\cite{wah2011caltech}, a 200-class fine-grained bird recognition dataset. Following the official Si-Blurry setting~\cite{moon2023online,kang2025advancing}, we set the disjoint class ratio to $m=50\%$ and the blurry sample ratio to $n=10\%$, and organize all classes into five learning phases. To examine robustness across different pretrained representations, we further consider ViT-B/16 backbones initialized with the Sup-21K~\cite{dosovitskiy2020image} and Sup-21K/1K~\cite{ridnik2021imagenet} checkpoints, together with a ViT-B/14 backbone initialized with DINOv2~\cite{oquab2024dinov2} for comprehensive evaluation.

\myPara{Metrics}
We evaluate continual performance using \emph{Final Average Accuracy} ($A_{\mathrm{Last}}$) and \emph{Average Anytime Accuracy} ($A_{\mathrm{AUC}}$). Let $a_{t,c}$ denote the test accuracy on class $c$ after evaluation step $t$, and let $\mathcal{C}_{t}$ denote the set of classes observed up to that step. The two metrics are defined as
\begin{align}
\label{eq:metric}
A_{\mathrm{Last}}
&=
\frac{1}{|\mathcal{C}_{T}|}
\sum_{c\in\mathcal{C}_{T}} a_{T,c},
\\
A_{\mathrm{AUC}}
&=
\frac{1}{T}
\sum_{t=1}^{T}
\frac{1}{|\mathcal{C}_{t}|}
\sum_{c\in\mathcal{C}_{t}} a_{t,c}.
\end{align}
$A_{\mathrm{Last}}$ evaluates the final performance after completing the continual stream, while $A_{\mathrm{AUC}}$ measures the average performance throughout the entire learning process and therefore reflects anytime continual performance. Higher values indicate better performance for both evaluation metrics.

\myPara{Baselines}
We consider four groups of baselines that differ in their use of pretrained knowledge and specialization for GCL. The first group comprises standard adaptation strategies: sequential full-network fine-tuning (Seq FT), linear probing with a frozen backbone, and slow-learner fine-tuning (Seq FT (SL))~\cite{zhang2023slca}. The second group includes representative PTM-based CL methods, namely L2P~\cite{wang2022learning}, DualPrompt~\cite{wang2022dualprompt}, and CODA-P~\cite{smith2023coda}. Following CODA-P, we implement L2P with prefix tuning (Deep L2P) to ensure architectural consistency. The third group consists of MVP~\cite{moon2023online} and MISA~\cite{kang2025advancing}, which are designed specifically for blurry GCL streams. Finally, we include the preliminary MePo~\cite{sun2026mepo} to directly quantify the improvements introduced by MePo++. For controlled comparisons, all PTM-based methods use length-5 prefixes in the first five transformer layers and are evaluated with identical pretrained checkpoints and downstream learners. %

\myPara{Implementation Details}
We follow the official MVP and MISA implementations~\cite{moon2023online,kang2025advancing}. MetaPrep uses 100 ImageNet-1K categories with 400 unlabeled images each; pseudo supervision comes only from feature clustering. Data are split into sequential meta-training and held-out joint subsets with $\gamma=0.3$. We optimize the encoder and auxiliary head with SGD ($\eta_{\theta}=1\times10^{-4}$, $\eta_{\psi}=1\times10^{-2}$, batch size 256) for 150, 100, and 150 epochs on Sup-21K, Sup-21K/1K, and DINOv2, respectively. The refined encoder yields pseudo-class prototypes and a fixed prior $\boldsymbol{\Sigma}_{\mathrm{pre}}$. StreamAlign estimates $\boldsymbol{\Sigma}_{\mathrm{cur}}^{t}$ per batch and adds $\epsilon\mathbf{I}$ ($\epsilon=1\times10^{-4}$) before Cholesky decomposition when needed. We use $\alpha=0.3$, $\lambda=0.002$, and $\tau=0.05$ by default. Downstream models use single-pass Adam (learning rate $5\times10^{-3}$, batch size 64). Experiments run on one RTX 3090 GPU and an AMD EPYC 7402 CPU at 2.8\,GHz; hyperparameter sensitivity appears in \cref{fig:param}.

\subsection{Comparisons with the State-of-the-Art}

We compare MePo++ with state-of-the-art methods under standard, data-scarce, and CLIP-based continual learning, and evaluate computational overhead under matched implementations and consistent evaluation protocols.

\myPara{Standard GCL Performance Comparison}
Under the standard GCL setting with full training data, \cref{tab:main} shows that MePo++ consistently improves different downstream learners across three datasets and multiple pretrained checkpoints. With Sup-21K, for example, MePo++ improves L2P by $7.92/7.15$ percentage points on CIFAR-100 and $7.63/9.94$ percentage points on ImageNet-R in terms of $A_{\mathrm{AUC}}/A_{\mathrm{Last}}$, corresponding to relative gains of $10.14\%/9.20\%$ and $18.00\%/26.05\%$, respectively. Similar improvements are observed for DualPrompt and MISA, and remain consistent under Sup-21/1K and DINOv2. Compared with our previous MePo~\cite{sun2026mepo}, MePo++ further improves most configurations, demonstrating the benefit of jointly refining pretrained representations before deployment and reconciling them throughout continual adaptation under evolving data streams and pretrained backbones.

\myPara{Few-Shot GCL Performance Comparison}
We also evaluate few-shot GCL with only $20\%$ of the training data. As shown in \cref{tab:fewshot}, MePo++ retains clear advantages under scarce observations. With Sup-21K, it improves DualPrompt by $9.14/16.02$ points on CIFAR-100 and $4.48/4.77$ points on ImageNet-R in $A_{\mathrm{AUC}}/A_{\mathrm{Last}}$, corresponding to relative gains of $17.53\%/29.90\%$ and $19.19\%/23.59\%$. Consistent gains across checkpoints and learners show that MePo++ exploits pretrained knowledge with limited supervision.

\myPara{CLIP-based Continual Learning Performance Comparison}
To examine whether the same representation-level strategy transfers beyond vision-only pretrained encoders, we further evaluate MePo++ on a CLIP backbone under the same Si-Blurry protocol, using representative prompt-based learners as downstream baselines. As shown in \cref{tab:vlm}, integrating MePo++ consistently improves both evaluated prompt-based learners on CIFAR-100, with gains in $A_{\mathrm{AUC}}/A_{\mathrm{Last}}$ of $0.84/1.30$ and $3.08/2.89$ percentage points for L2P and DualPrompt, respectively. The same trend holds on ImageNet-R, where MePo++ improves L2P by $0.73/0.54$ points and DualPrompt by $2.92/3.34$ points. These consistent paired gains indicate that the proposed representation refinement and reconciliation strategy remains beneficial with vision-language pretraining and transfers across downstream prompt learners with different adaptation mechanisms.

\myPara{Computational Performance Comparison}
\cref{tab:complexity} evaluates MePo++ overhead on CUB-200 with Sup-21K. With DualPrompt, learnable parameters rise from $637$K ($0.74\%$) to $1{,}213$K ($1.41\%$), and batch time from $4.87$\,s to $5.09$\,s, while accuracy improves from $55.73$ to $59.05$. With MISA, the parameter ratio changes from $0.76\%$ to $1.43\%$ and time from $4.84$\,s to $5.05$\,s, while accuracy rises from $65.40$ to $69.52$. MePo++ adds no learnable parameters over MePo and incurs marginal runtime, yet improves accuracy by $0.69$ and $1.39$ points for DualPrompt and MISA.

\begin{table}[t]
\centering
\caption{
Ablation study of the core modules using Sup-21K.
}
\label{tab:ablation}
\setlength{\tabcolsep}{1.4pt}
\subcaption{CIFAR-100}
\label{tab:ablation-a}
\resizebox{\linewidth}{!}{%
\begin{tabular}{
    cc
    *{4}{
        S[
            table-format=2.2,
            table-space-text-post={\std{11.55}}
        ]
    }
}
    \toprule
        \multirow{2.5}{*}{\textbf{MetaPrep}}
        &
        \multirow{2.5}{*}{\textbf{StreamAlign}}
        &
        \multicolumn{2}{c}{\textbf{DualPrompt}}
        &
        \multicolumn{2}{c}{\textbf{MISA}}
        \\
        \cmidrule(lr){3-4}
        \cmidrule(lr){5-6}
        &
        &
        {$A_{\mathrm{AUC}}~(\uparrow)$}
        &
        {$A_{\mathrm{Last}}~(\uparrow)$}
        &
        {$A_{\mathrm{AUC}}~(\uparrow)$}
        &
        {$A_{\mathrm{Last}}~(\uparrow)$}
        \\
        \midrule
        \xmark & \xmark
        & 66.36\std{4.42}
        & 58.09\std{4.40}
        & 80.35\std{2.39}
        & 80.75\std{1.24}
        \\

        \cmark & \xmark
        & 70.65\std{3.25}
        & 65.27\std{4.15}
        & 81.29\std{2.28}
        & 81.84\std{1.00}
        \\

        \xmark & \cmark
        & 69.10\std{5.22}
        & 65.76\std{5.59}
        & 81.36\std{2.46}
        & 82.19\std{0.78}
        \\

        \cmark & \cmark
        & \bfseries 72.99\std{4.11}
        & \bfseries 72.00\std{2.85}
        & \bfseries 82.73\std{0.84}
        & \bfseries 84.01\std{0.50}
        \\
        \bottomrule
\end{tabular}%
}
\vspace{0.2mm}
\subcaption{ImageNet-R}
\label{tab:ablation-b}
\resizebox{\linewidth}{!}{%
\begin{tabular}{
    cc
    *{4}{
        S[
            table-format=2.2,
            table-space-text-post={\std{11.55}}
        ]
    }
}
    \toprule
        \multirow{2.5}{*}{\textbf{MetaPrep}}
        &
        \multirow{2.5}{*}{\textbf{StreamAlign}}
        &
        \multicolumn{2}{c}{\textbf{DualPrompt}}
        &
        \multicolumn{2}{c}{\textbf{MISA}}
        \\
        \cmidrule(lr){3-4}
        \cmidrule(lr){5-6}
        &
        &
        {$A_{\mathrm{AUC}}~(\uparrow)$}
        &
        {$A_{\mathrm{Last}}~(\uparrow)$}
        &
        {$A_{\mathrm{AUC}}~(\uparrow)$}
        &
        {$A_{\mathrm{Last}}~(\uparrow)$}
        \\
        \midrule
        \xmark & \xmark
        & 38.63\std{2.19}
        & 30.71\std{0.82}
        & 51.52\std{2.09}
        & 45.08\std{1.43}
        \\

        \cmark & \xmark
        & 43.14\std{2.18}
        & 34.93\std{1.50}
        & 52.63\std{1.95}
        & 46.03\std{1.21}
        \\

        \xmark & \cmark
        & 42.47\std{1.75}
        & 36.92\std{1.02}
        & 52.61\std{1.74}
        & 48.01\std{2.05}
        \\

        \cmark & \cmark
        & \bfseries 47.34\std{1.63}
        & \bfseries 40.77\std{1.84}
        & \bfseries 55.18\std{1.72}
        & \bfseries 50.20\std{1.43}
        \\
        \bottomrule
\end{tabular}%
}
\vspace{0.2mm}
\subcaption{CUB-200}
\label{tab:ablation-c}
\resizebox{\linewidth}{!}{%
\begin{tabular}{
    cc
    *{4}{
        S[
            table-format=2.2,
            table-space-text-post={\std{11.55}}
        ]
    }
}
    \toprule
        \multirow{2.5}{*}{\textbf{MetaPrep}}
        &
        \multirow{2.5}{*}{\textbf{StreamAlign}}
        &
        \multicolumn{2}{c}{\textbf{DualPrompt}}
        &
        \multicolumn{2}{c}{\textbf{MISA}}
        \\
        \cmidrule(lr){3-4}
        \cmidrule(lr){5-6}
        &
        &
        {$A_{\mathrm{AUC}}~(\uparrow)$}
        &
        {$A_{\mathrm{Last}}~(\uparrow)$}
        &
        {$A_{\mathrm{AUC}}~(\uparrow)$}
        &
        {$A_{\mathrm{Last}}~(\uparrow)$}
        \\
        \midrule
        \xmark & \xmark
        & 55.73\std{2.77}
        & 47.08\std{4.94}
        & 65.40\std{3.01}
        & 60.20\std{1.82}
        \\

        \cmark & \xmark
        & 56.67\std{2.66}
        & 48.56\std{5.07}
        & 66.39\std{2.97}
        & 60.83\std{1.57}
        \\

        \xmark & \cmark
        & 57.51\std{2.86}
        & 52.88\std{2.67}
        & 69.04\std{3.14}
        & 67.60\std{0.60}
        \\

        \cmark & \cmark
        & \bfseries 59.05\std{2.98}
        & \bfseries 53.87\std{3.45}
        & \bfseries 69.52\std{3.00}
        & \bfseries 68.63\std{1.39}
        \\
        \bottomrule
\end{tabular}%
}
\end{table}

\begin{table}[t]
    \centering
    \caption{
        Effect of different representation preparation strategies on ImageNet-R using the Sup-21K pretrained backbone with DualPrompt and MISA.
    }
    \label{tab:post_ablation}
    \setlength{\tabcolsep}{3.2pt}
    \renewcommand{\arraystretch}{1.08}

    \resizebox{\linewidth}{!}{%
    \begin{tabular}{
        l
        *{4}{
            S[
                table-format=2.2,
                table-space-text-post={\std{11.55}}
            ]
        }
    }
        \toprule
        \multirow{2.5}{*}{\textbf{Post-training Strategy}}
        &
        \multicolumn{2}{c}{\textbf{DualPrompt}}
        &
        \multicolumn{2}{c}{\textbf{MISA}}
        \\
        \cmidrule(lr){2-3}
        \cmidrule(lr){4-5}
        &
        {$A_{\mathrm{AUC}}~(\uparrow)$}
        &
        {$A_{\mathrm{Last}}~(\uparrow)$}
        &
        {$A_{\mathrm{AUC}}~(\uparrow)$}
        &
        {$A_{\mathrm{Last}}~(\uparrow)$}
        \\
        \midrule

        w/o Post-training
        & 42.47\std{1.75}
        & 36.92\std{1.02}
        & 52.61\std{1.74}
        & 48.01\std{2.05}
        \\
        
        w/ Post-training (OML~\cite{javed2019meta})
        & 46.86\std{2.14}
        & 39.11\std{1.93}
        & 54.54\std{2.67}
        & 49.29\std{0.98}
        \\
        
        w/ Post-training (MePo~\cite{sun2026mepo})
        & 45.48\std{1.63}
        & 39.09\std{1.31}
        & 54.10\std{2.16} 
        & 50.05\std{1.84}
        \\

        w/ Post-training (MePo++)
        & \bfseries 47.34\std{1.63}
        & \bfseries 40.77\std{1.84}
        & \bfseries 55.18\std{1.72}
        & \bfseries 50.20\std{1.43}
        \\

        \bottomrule
    \end{tabular}%
    }
\end{table}

\subsection{Ablation Study}

We further dissect MePo++ by examining the contribution of its components and replacing each with alternative designs while keeping the remaining framework unchanged to isolate their individual effects under the same evaluation protocol.

\myPara{Effect of Core Modules}
We first evaluate the individual contributions of MetaPrep and StreamAlign across CIFAR-100, ImageNet-R, and CUB-200 using both DualPrompt and MISA. As shown in \cref{tab:ablation}, either component alone generally improves the corresponding baseline, while their combination consistently performs best. On ImageNet-R, MetaPrep improves DualPrompt by $4.51/4.22$ percentage points in $A_{\mathrm{AUC}}/A_{\mathrm{Last}}$ ($11.67\%/13.74\%$ relative), while StreamAlign yields gains of $3.84/6.21$ points ($9.94\%/20.22\%$). Combining both further increases the gains to $8.71/10.06$ points ($22.55\%/32.76\%$). A similar complementary effect is observed with MISA, where the complete model achieves the strongest performance. The same trend holds on CIFAR-100 and CUB-200, confirming that MetaPrep and StreamAlign address complementary aspects of PTM-based GCL by improving plasticity and stability, respectively.

\myPara{Effect of Alternative Designs for Representation Preparation}
To examine whether the gain of MetaPrep simply results from additional upstream adaptation, we keep StreamAlign unchanged and replace MetaPrep with alternative post-training strategies. As shown in \cref{tab:post_ablation}, compared with no representation preparation, MetaPrep improves DualPrompt by $4.87/3.85$ percentage points in $A_{\mathrm{AUC}}/A_{\mathrm{Last}}$ ($11.47\%/10.43\%$ relative), and MISA by $2.57/2.19$ points ($4.89\%/4.56\%$ relative). OML~\cite{javed2019meta} and the preliminary MePo refinement also yield consistent improvements, confirming the general benefit of preparing pretrained representations before deployment. MetaPrep nevertheless performs best across all metrics, indicating that the improvement stems not merely from additional post-training, but from its GCL-oriented pseudo continual sequence construction and meta-refinement procedure.

\myPara{Effect of Alternative Designs for Representation Reconciliation}
To isolate the design of StreamAlign, we fix the representation preparation stage and replace the downstream reconciliation mechanism with alternative alignment strategies. As shown in \cref{tab:out_ablation}, StreamAlign consistently outperforms no alignment, mask-based alignment, and the covariance-only alignment used in preliminary MePo. Compared with covariance-only alignment, StreamAlign further improves DualPrompt and MISA, indicating that geometric matching alone is insufficient. Explicitly reconciling the reconstructed and online representations at the semantic level provides additional gains under evolving streams, validating the joint geometric and semantic design of StreamAlign for maintaining both stability and discriminability across downstream tasks.

\begin{table}[t]
    \centering
    \caption{
        Effect of different representation reconciliation strategies on ImageNet-R using the Sup-21K pretrained backbone with DualPrompt and MISA.
    }
    \label{tab:out_ablation}
    \setlength{\tabcolsep}{3.2pt}
    \renewcommand{\arraystretch}{1.08}

    \resizebox{\linewidth}{!}{%
    \begin{tabular}{
        l
        *{4}{
            S[
                table-format=2.2,
                table-space-text-post={\std{11.55}}
            ]
        }
    }
        \toprule
        \multirow{2.5}{*}{\textbf{Output Alignment Strategy}}
        &
        \multicolumn{2}{c}{\textbf{DualPrompt}}
        &
        \multicolumn{2}{c}{\textbf{MISA}}
        \\
        \cmidrule(lr){2-3}
        \cmidrule(lr){4-5}
        &
        {$A_{\mathrm{AUC}}~(\uparrow)$}
        &
        {$A_{\mathrm{Last}}~(\uparrow)$}
        &
        {$A_{\mathrm{AUC}}~(\uparrow)$}
        &
        {$A_{\mathrm{Last}}~(\uparrow)$}
        \\
        \midrule

        w/o Alignment
        & 43.14\std{2.18}
        & 34.93\std{1.50}
        & 43.44\std{1.80}
        & 33.78\std{1.63}
        \\

        w/ Mask Alignment (MISA~\cite{kang2025advancing})
        & 44.33\std{2.38}
        & 35.01\std{1.12}
        & 52.63\std{1.95}
        & 46.03\std{1.21}
        \\

        w/ Covariance Alignment (MePo~\cite{sun2026mepo})
        & 44.29\std{2.14}
        & 36.26\std{1.32}
        & 53.26\std{2.32}
        & 47.81\std{1.56}
        \\

        w/ StreamAlign (MePo++)
        & \bfseries 47.34\std{1.63}
        & \bfseries 40.77\std{1.84}
        & \bfseries 55.18\std{1.72}
        & \bfseries 50.20\std{1.43}
        \\

        \bottomrule
    \end{tabular}%
    }
\end{table}

\begin{figure}
    \centering
    \includegraphics[width=\linewidth]{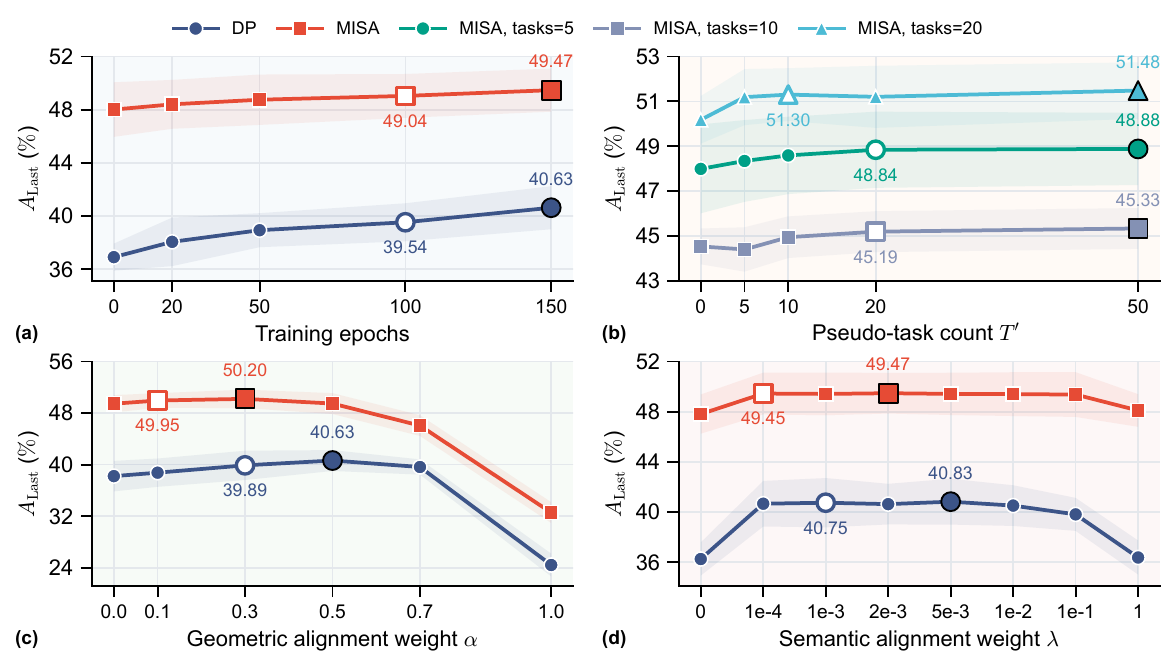}
    \caption{\textbf{Sensitivity analysis of key hyperparameters.} (a) Number of MetaPrep training epochs. (b) Number of pseudo-tasks $T'$ under different stream granularities. (c) Geometric alignment weight $\alpha$. (d) Semantic alignment weight $\lambda$ during online adaptation.}
    \label{fig:param}
    \phantomsubcaption\label{fig:param-a}
    \phantomsubcaption\label{fig:param-b}
    \phantomsubcaption\label{fig:param-c}
    \phantomsubcaption\label{fig:param-d}
\end{figure}

\myPara{Effect of Key Hyperparameters}
\cref{fig:param-a,fig:param-b} analyze the two key hyperparameters of MetaPrep. As shown in \cref{fig:param-a}, increasing the number of meta-training epochs steadily improves $A_{\mathrm{Last}}$, with both DualPrompt and MISA approaching saturation around 100--150 epochs. \cref{fig:param-b} further shows that increasing the pseudo-task count $T'$ generally improves performance before reaching a stable plateau, and this trend remains consistent across different stream granularities. For StreamAlign, \cref{fig:param-c} shows that the geometric alignment weight $\alpha$ achieves the best trade-off around $0.3$--$0.5$; overly aggressive alignment, particularly $\alpha=1$, substantially degrades performance, indicating that fully replacing the online representation with its aligned counterpart suppresses useful plastic information. As shown in \cref{fig:param-d}, the semantic alignment weight $\lambda$ is relatively robust over a broad range, with the best performance around $2\times10^{-3}$--$5\times10^{-3}$, while an excessively large weight again harms adaptation. Overall, MePo++ remains stable across reasonable hyperparameter ranges, supporting a moderate balance among representation refinement, geometric alignment, and semantic reconciliation.

\subsection{Qualitative and Quantitative Analysis}

We further analyze MePo++ from representation, optimization, and prediction perspectives. These analyses provide complementary evidence for how MePo++ improves continual adaptation beyond aggregate accuracy metrics.

\myPara{Overall Representation Quality}
\Cref{fig:representation_visualization} compares the feature spaces learned by MISA and MePo++ across incremental stages T0--T4. MePo++ consistently forms more compact class clusters and clearer inter-class boundaries, indicating a better organized representation space throughout the continual stream. This is quantitatively supported by the separation score, which improves from $0.6786$, $0.4889$, $0.4714$, $0.3314$, and $0.5594$ with MISA to $1.0013$, $0.8574$, $1.0388$, $0.8688$, and $0.8567$ with MePo++, respectively. The gains remain clear at later stages, where blurry observations and newly introduced classes make representation maintenance more challenging. This behavior is consistent with the design of MePo++: MetaPrep improves continual adaptability before deployment, while StreamAlign stabilizes the evolving feature space during online adaptation, together yielding more discriminative representations across the stream.

\begin{figure}
    \centering
    \includegraphics[width=\linewidth]{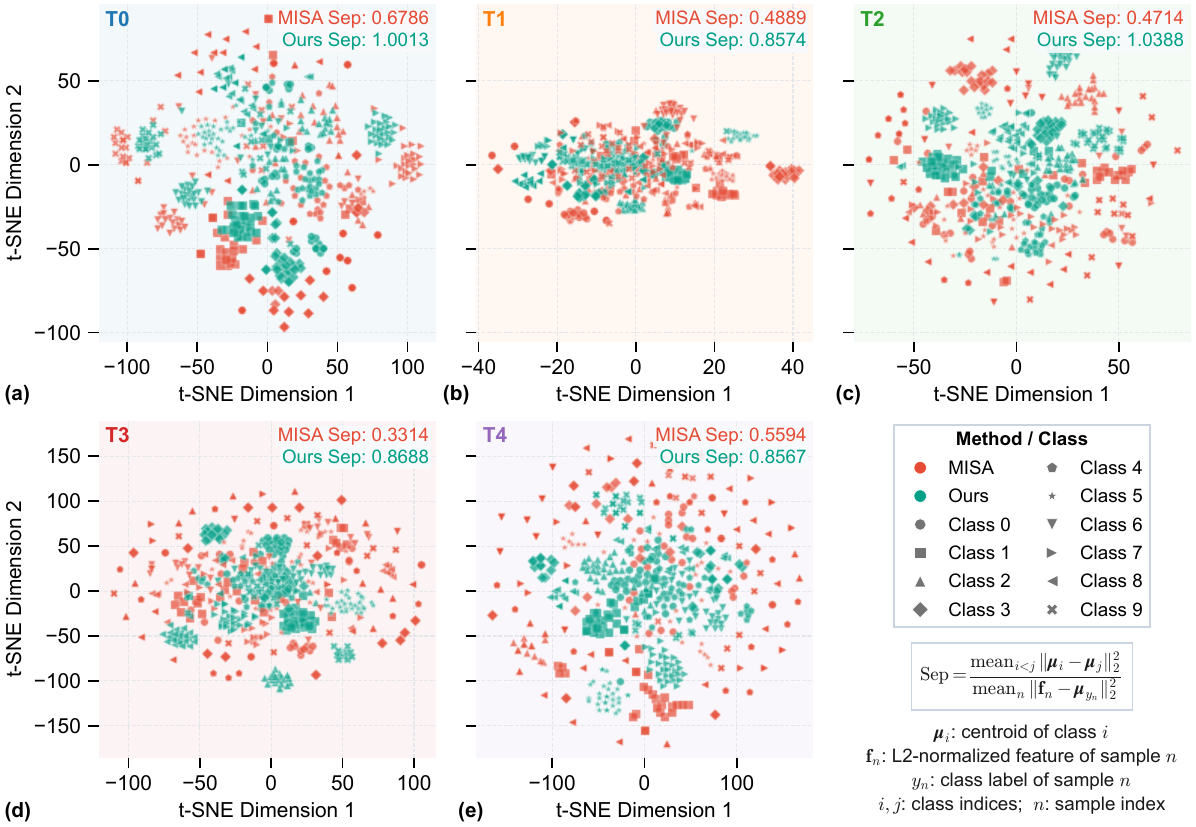}
    \caption{\textbf{Representation quality across incremental tasks.} t-SNE visualizations of representations learned by MISA and MePo++ from T0 to T4. Marker shapes denote different classes, and higher separation scores indicate better class-level discriminability.}
    \label{fig:representation_visualization}
\end{figure}

\myPara{Sparsity of Refined Representations}
\cref{fig:sparse_activation} compares the feature activation patterns of MISA and MePo++ across incremental tasks T0--T4. MePo++ consistently activates substantially fewer feature dimensions, with active rates reduced from $58.07\%$, $73.18\%$, $48.57\%$, $70.44\%$, and $43.62\%$ under MISA to $20.96\%$, $7.16\%$, $8.98\%$, $13.67\%$, and $13.80\%$, respectively. Instead of distributing responses broadly across dimensions, MePo++ concentrates activation on a smaller subset of informative features, yielding a sparser and more selective representation space. Such representations are particularly desirable in CL, where sparse and disentangled feature structures have been shown to reduce representational overlap and interference across classes~\cite{michieli2021continual,pourcel2022online}. Meta-learning for continual adaptation has also been observed to naturally induce sparse representations that facilitate subsequent online learning~\cite{javed2019meta}, while feature decorrelation provides a complementary means of reducing class interference and catastrophic forgetting~\cite{shi2022mimicking}. The observed activation patterns are therefore consistent with the role of MetaPrep in refining pretrained representations for sequential adaptation, providing a cleaner and less interfering representation basis for downstream CL.

\begin{figure}
    \centering
    \includegraphics[width=\linewidth]{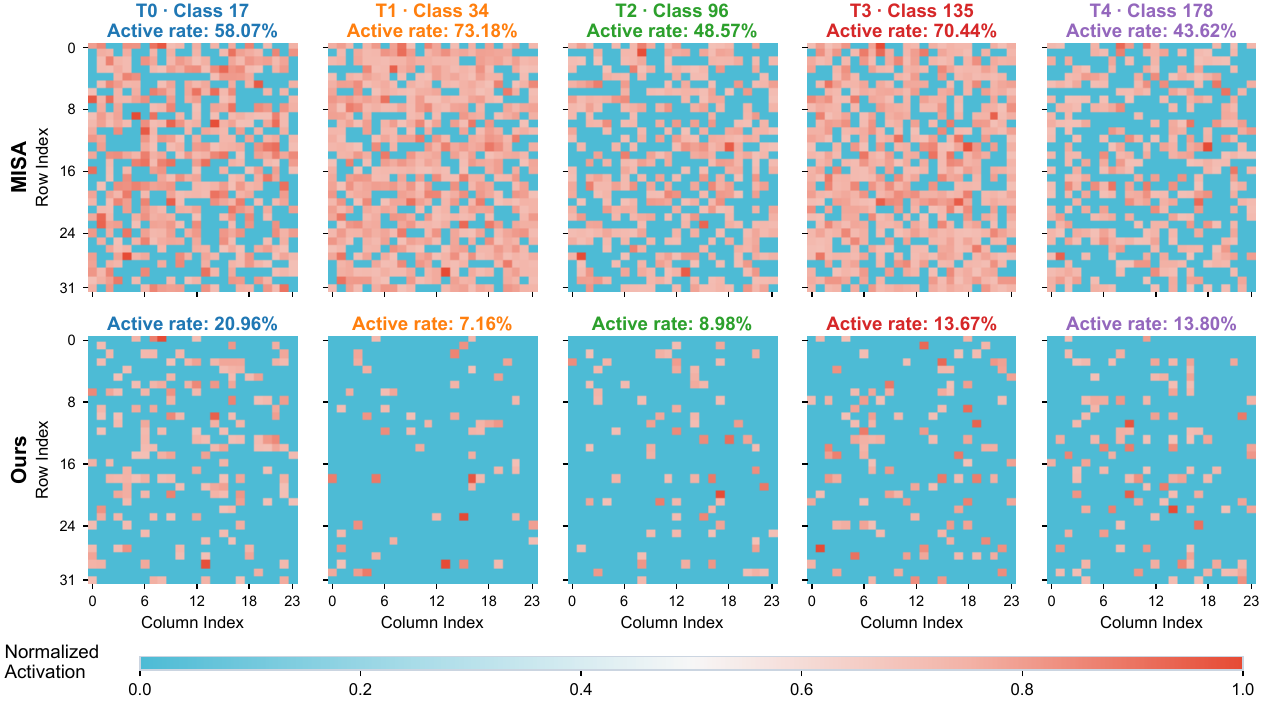}
    \caption{\textbf{Sparse feature activations across incremental tasks.} Feature activation patterns of MISA (top) and MePo++ (bottom) from T0 to T4. Activations below $0.7$ are suppressed. MePo++ consistently exhibits lower active rates, indicating sparser representations.}
    \label{fig:sparse_activation}
\end{figure}

\begin{figure}
    \centering
    \includegraphics[width=\linewidth]{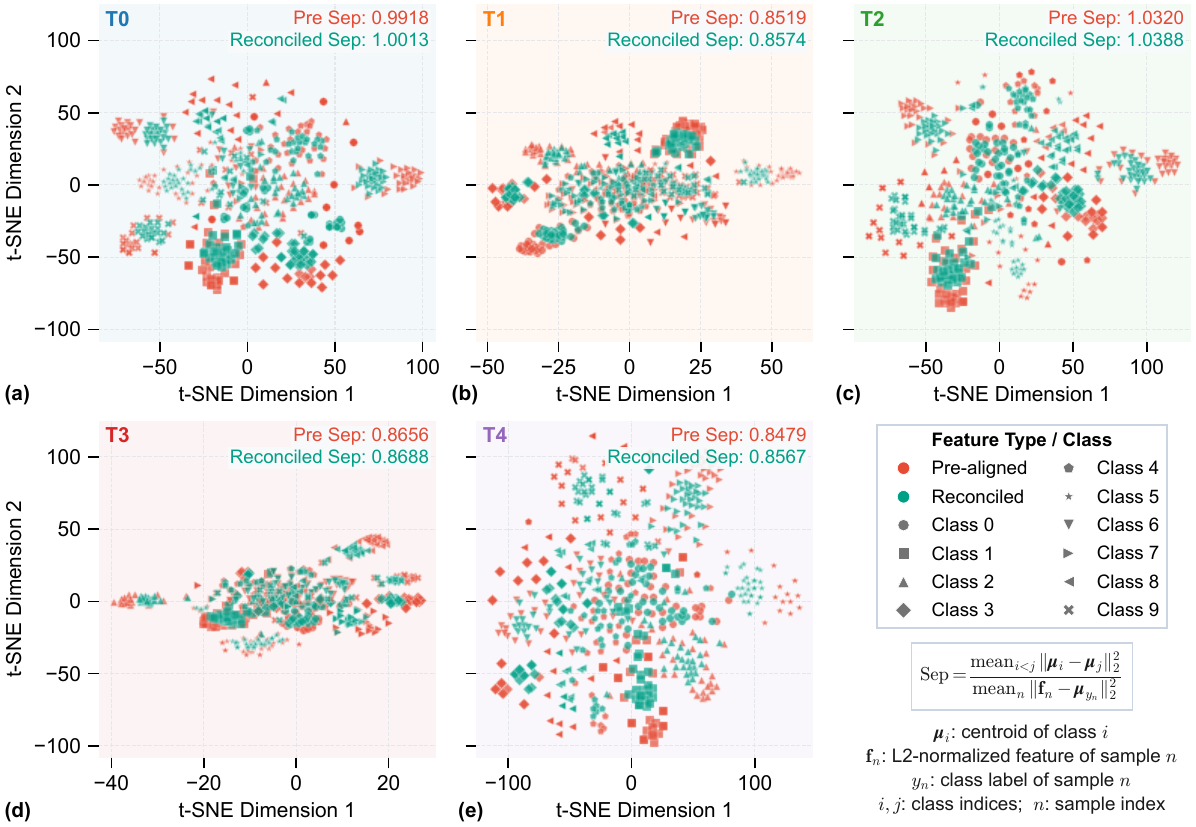}
    \caption{\textbf{Effect of representation reconciliation across incremental tasks.} Overlaid t-SNE visualizations of pre-aligned and reconciled features from T0 to T4. Marker shapes denote classes, colors distinguish feature types, and higher separation scores indicate greater class-level discriminability across tasks.}
    \label{fig:feature_evolution}
\end{figure}

\myPara{Effect of Representation Reconciliation}
To isolate the contribution of StreamAlign, \cref{fig:feature_evolution} overlays the pre-aligned features $\mathbf{f}_{i}$ and reconciled features $\mathbf{f}_{\mathrm{trans},i}$. Across T0--T4, reconciliation preserves the global semantic layout while introducing localized class-conditional adjustments. The separation score increases from $0.9918$ to $1.0013$, $0.8519$ to $0.8574$, $1.0320$ to $1.0388$, $0.8656$ to $0.8688$, and $0.8479$ to $0.8567$, respectively. These consistent absolute gains of $0.0032$--$0.0095$, with the largest improvements at T0 and T4, indicate systematic refinement rather than wholesale restructuring of the online feature space. This behavior supports StreamAlign: geometry-guided reconstruction anchors evolving representations to a stable upstream reference, while semantic reconciliation preserves adaptive class structure throughout the stream.

\myPara{Loss Landscape Analysis}
To further examine the optimization characteristics of MePo++, \cref{fig:loss_landscape} visualizes the loss landscape around the converged solutions of DualPrompt and MePo++. Compared with DualPrompt, MePo++ exhibits a broader low-loss region with more widely spaced contours, indicating a flatter local optimum and reduced sensitivity to parameter perturbations. This suggests that the representation refined by MetaPrep and stabilized by StreamAlign leads to a more robust optimization geometry, which is consistent with the improved stability and lower forgetting observed throughout continual adaptation to evolving data streams.

\begin{figure}
    \centering
    \includegraphics[width=\linewidth]{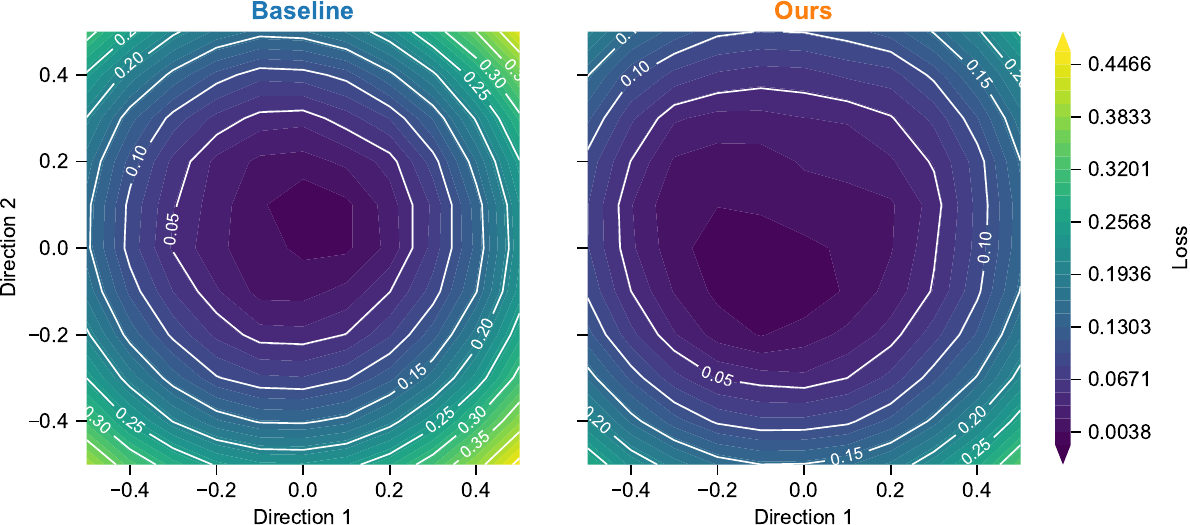}
    \caption{\textbf{Loss landscape comparison between DualPrompt (left) and MePo++ (right).} MePo++ exhibits a broader and flatter low-loss region around the converged solution, indicating reduced sensitivity to parameter perturbations during continual adaptation.}
    \label{fig:loss_landscape}
\end{figure}

\myPara{Case Study}
\cref{fig:case_study} provides a fine-grained comparison between MISA and MePo++ at both the prediction and confidence levels on ImageNet-R. Overall, MePo++ achieves higher accuracy than MISA ($48.3\%$ vs.\ $45.6\%$), with $10.3\%$ of the test samples corrected exclusively by MePo++ compared with $7.6\%$ corrected exclusively by MISA. More importantly, among the samples misclassified by both methods, MePo++ assigns a higher probability to the ground-truth class in $22.5\%$ of all test samples, indicating that its predictions are often shifted closer to the correct semantic category even when the final decision remains incorrect. \cref{fig:case_study-b,fig:case_study-c} further show that MePo++ maintains substantially higher ground-truth confidence for earlier classes across subsequent evaluation tasks, whereas MISA exhibits pronounced confidence decay. The examples in \cref{fig:case_study-d,fig:case_study-e,fig:case_study-f} illustrate successful corrections with substantially increased ground-truth confidence, while \cref{fig:case_study-g} shows a failure case where MePo++ still raises the ground-truth probability despite making the same incorrect prediction. These observations provide sample-level evidence that MePo++ not only improves accuracy but also preserves reliable semantic evidence for previously learned classes as the continual stream evolves.

\begin{figure}
    \centering
    \includegraphics[width=\linewidth]{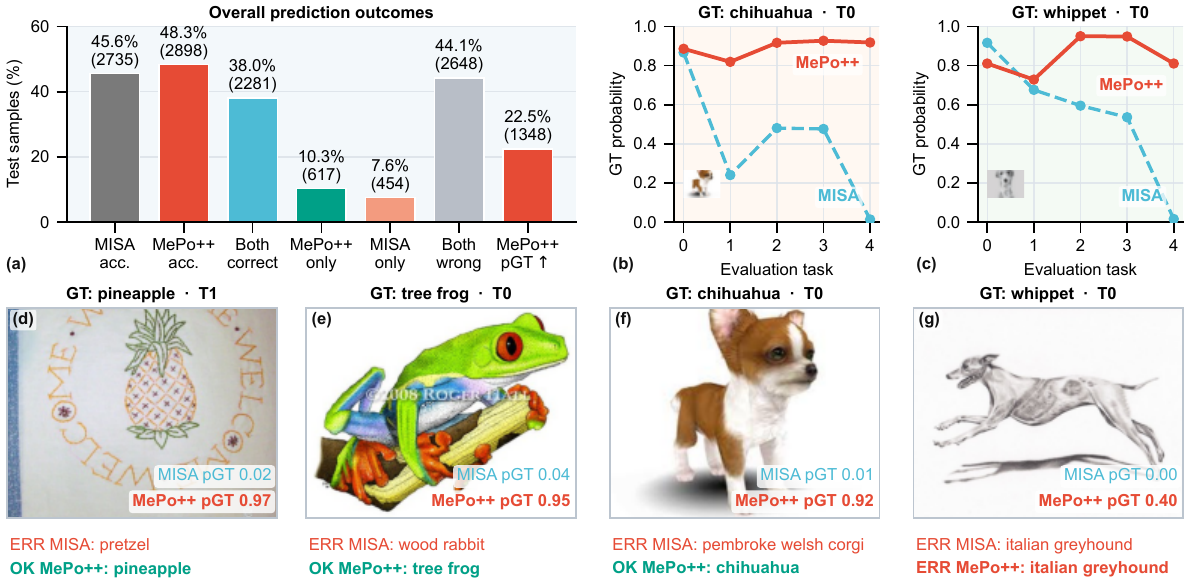}
    \caption{\textbf{Case study of prediction outcomes and ground-truth confidence.} (a) Overall prediction outcomes of MISA and MePo++; ``MePo++ pGT$\uparrow$'' denotes samples misclassified by both methods for which MePo++ assigns a higher probability to the ground-truth class than MISA. (b--c) Evolution of ground-truth probability for representative classes across evaluation tasks. (d--f) Examples corrected by MePo++ with substantially higher ground-truth confidence. (g) A failure case where both methods are incorrect, while MePo++ still assigns a higher probability to the ground-truth class.}
    \label{fig:case_study}
    \phantomsubcaption\label{fig:case_study-a}
    \phantomsubcaption\label{fig:case_study-b}
    \phantomsubcaption\label{fig:case_study-c}
    \phantomsubcaption\label{fig:case_study-d}
    \phantomsubcaption\label{fig:case_study-e}
    \phantomsubcaption\label{fig:case_study-f}
    \phantomsubcaption\label{fig:case_study-g}
\end{figure}

\section{Conclusion}

We revisited PTM-based GCL from a unified representation perspective and identified two challenges: upstream-downstream misalignment and a downstream alignment gap caused by unreliable statistics in blurry streams. MePo++ addresses them through MetaPrep, which refines representations before deployment via pseudo-sequence construction and bi-level meta-refinement, and StreamAlign, which reconciles online features with a stable geometry prior. Experiments across datasets, PTMs, and learners show consistent gains in standard, few-shot, and CLIP-based continual learning. Ablations and representation analyses further show more selective features and preserved class structure throughout evolving streams.

\myPara{Limitations and Future Work}
MePo++ has two limitations. First, MetaPrep introduces an additional pre-deployment stage and requires unlabeled upstream data, increasing the preparation cost for large models. Second, StreamAlign relies on a fixed geometry prior that provides a stable reference under blurry streams but may become less representative under severe or long-term distribution shifts. Future work could reduce refinement cost and adapt geometry priors using reliable downstream evidence while preserving upstream knowledge.

\begin{IEEEbiography}[{\includegraphics[width=1in,height=1.1in,clip,keepaspectratio]{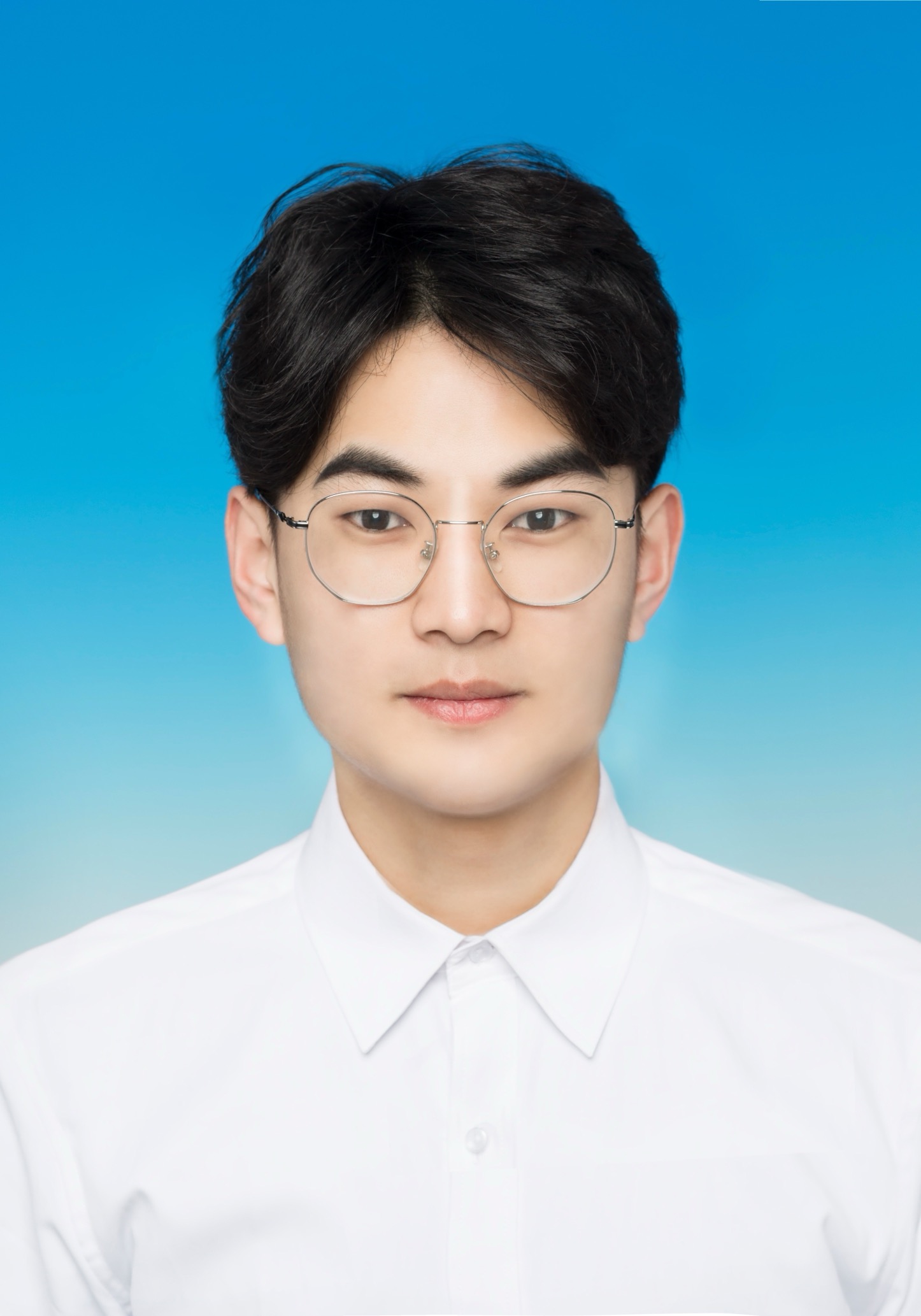}}]{Guanglong Sun}
    is currently pursuing his Ph.D. degree in Biology at the School of Life Sciences and IDG/McGovern Institute for Brain Research, Tsinghua University (advised by Prof. Jun Zhu and Prof. Yi Zhong). His research focuses on NeuroAI, which explores how principles of biological learning and memory can inspire more adaptive and generalizable artificial intelligence systems. He has published papers in major conferences and journals in related fields, including Patterns, ICML, ICLR, ECCV, IROS, etc.
\end{IEEEbiography}
\begin{IEEEbiography}[{\includegraphics[width=1in,height=1.1in,clip,keepaspectratio]{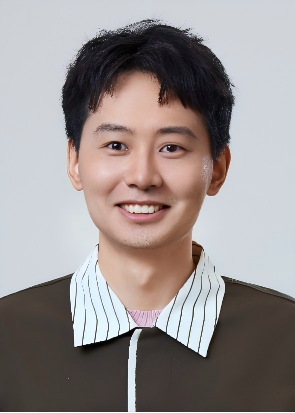}}]{Kanglei Zhou}
    is currently a Postdoctoral Researcher with the Department of Psychological and Cognitive Science, Tsinghua University. He received his Ph.D. degree in Computer Science and Engineering from Beihang University in 2025 and his B.E. degree from Henan Normal University in 2020. In 2024, he was a Visiting Student at Durham University. His research interests include human motion analysis and continual learning, with publications in leading venues such as IJCV, TIP, TVCG, ICML, ICLR, CVPR, and ECCV.
\end{IEEEbiography}
\begin{IEEEbiography}[{\includegraphics[width=1in,height=1.1in,clip,keepaspectratio]{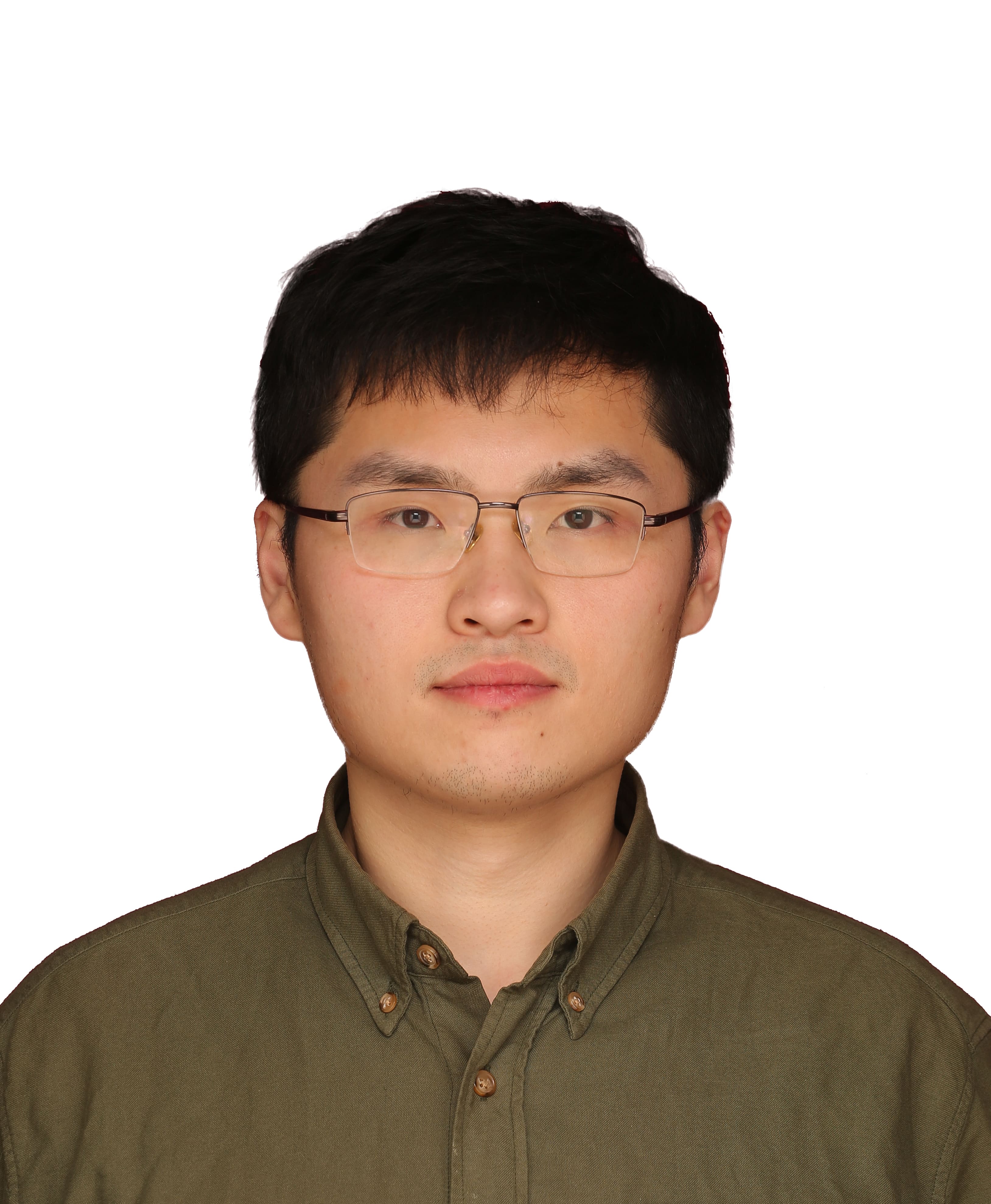}}]{Liyuan Wang} (Member, IEEE)
    is an Assistant Professor at the Department of Psychological and Cognitive Sciences, Tsinghua University. He received the B.S. and Ph.D. degrees from Tsinghua University, where he also conducted his postdoctoral research. His work on continual learning has been published in major conferences and journals in related fields, such as Nature Machine Intelligence, Nature Communications, TPAMI, Patterns, NeurIPS, ICLR, ICML, etc.
\end{IEEEbiography} 
\begin{IEEEbiography}[{\includegraphics[width=1in,height=1.1in,clip,keepaspectratio]{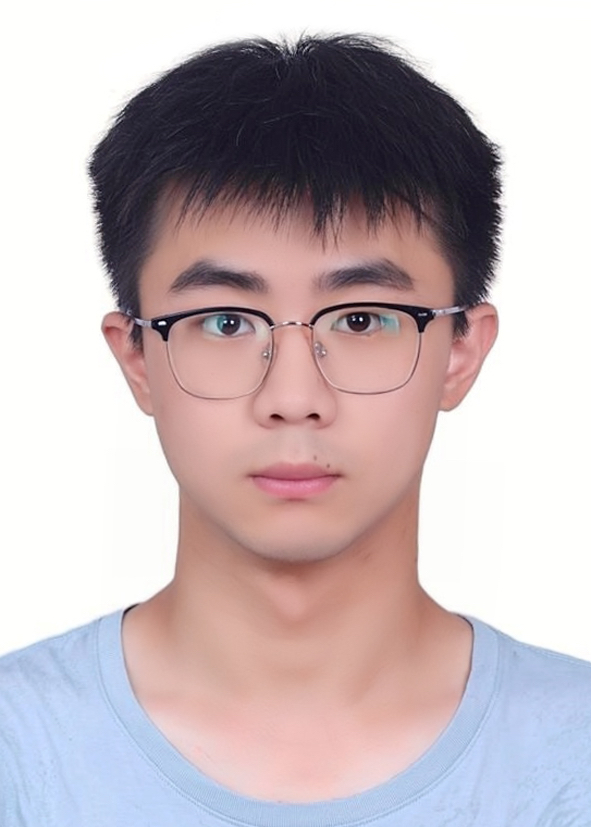}}]{Qi Cheng}
    is currently an undergraduate student at Shanghai Jiao Tong University. His early research focused on computer vision and multimodal learning. His current research interests center on Brain4AI, particularly brain-inspired artificial intelligence, continual learning, memory systems, and computational neuroscience.
\end{IEEEbiography}
\begin{IEEEbiography}[{\includegraphics[width=1in,height=1.1in,clip,keepaspectratio]{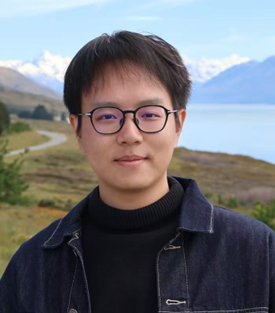}}]{Hongwei Yan}
    received the B.S. degree in computer science from the Institute for Interdisciplinary Information Sciences (Yao's Class), Tsinghua University, Beijing, China, in 2022. He is currently pursuing the Ph.D. degree in biology at the School of Life Science, Tsinghua University, where his research is centered on brain-inspired artificial intelligence. His primary research interests cover continual learning, domain generalization, multi-modal foundation models. He has published multiple papers in major AI conferences including CVPR, ICML and ICLR.
\end{IEEEbiography} 
\begin{IEEEbiography}[{\includegraphics[width=1in,height=1.1in,clip,keepaspectratio]{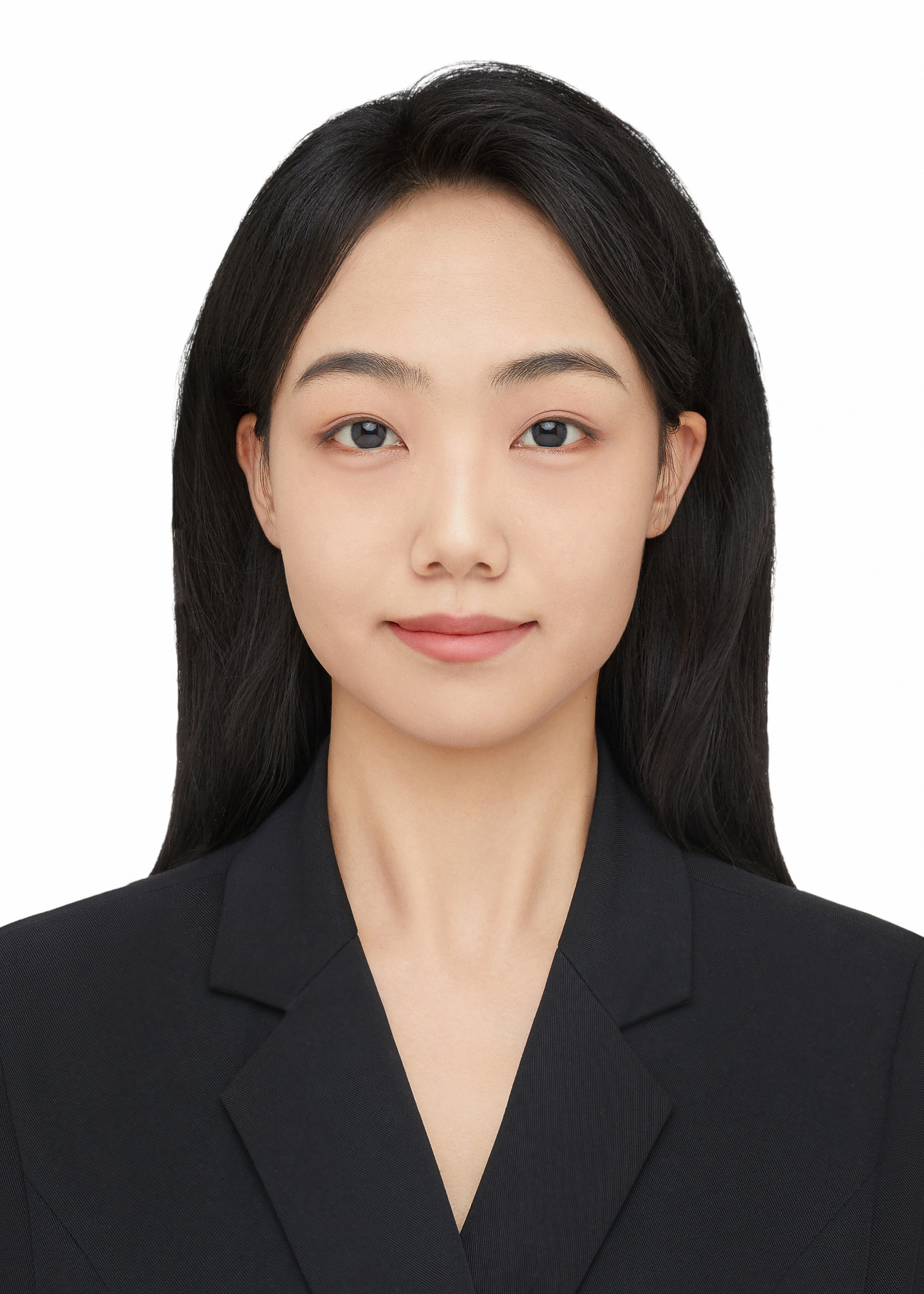}}]{Shuang Cui}
    received the B.S. degree from Northwest A\&F University, Xianyang, China, in 2021. She is currently pursuing the Ph.D. degree at the University of Chinese Academy of Sciences and the Institute of Software, Chinese Academy of Sciences, Beijing, China. Her research interests include transfer learning, test-time adaptation, and continual learning.
\end{IEEEbiography} 
\begin{IEEEbiography}
[{\includegraphics[width=1in,height=1.1in,clip,keepaspectratio]{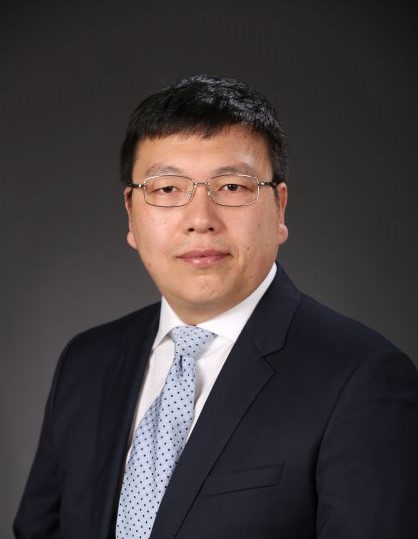}}] {Hang Su} (Member, IEEE) 
    is an associated professor with the Department of Computer Science and Technology, Tsinghua University. His research interests lie in the adversarial machine learning and robust computer vision, based on which he has published more than 50 papers including CVPR, ECCV, IEEE Transactions on Medical Imaging, etc.
    He received ``Young Investigator Award'' from MICCAI2012, the ``Best Paper Award'' in AVSS2012, and ``Platinum Best Paper Award'' in ICME2018.
\end{IEEEbiography}
\begin{IEEEbiography}
[{\includegraphics[width=1in,height=1.1in,clip,keepaspectratio]{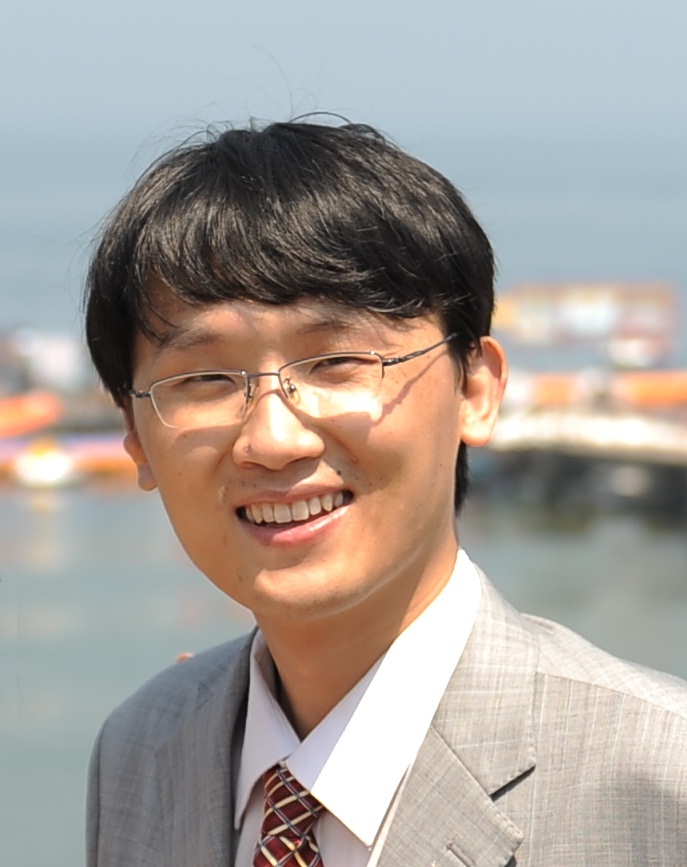}}]{Jun Zhu} (Fellow, IEEE)
    received his B.S. and Ph.D. degrees from the Department of Computer Science and Technology in Tsinghua University, where he is currently a Bosch AI Professor. He was an adjunct faculty and postdoctoral fellow in the Machine Learning Department, Carnegie Mellon University. His research interest is primarily on developing machine learning methods to understand scientific and engineering data arising from various fields. He regularly serves as senior Area Chairs and Area Chairs at prestigious conferences, including ICML, NeurIPS, ICLR, IJCAI and AAAI. He was selected as ``AI's 10 to Watch'' by IEEE Intelligent Systems. He is a Fellow of the IEEE and an associate editor-in-chief of IEEE TPAMI. 
\end{IEEEbiography}
\begin{IEEEbiography}
[{\includegraphics[width=1in,height=1.1in,clip,keepaspectratio]{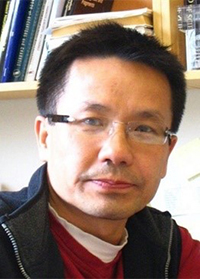}}]{Yi Zhong}
    is a Professor at the School of Life Sciences and IDG/McGovern Institute for Brain Research, Tsinghua University. His research focuses on the biological mechanisms and theoretical principles of learning and memory, as well as their applications in brain-inspired artificial intelligence. His research has been published in major conferences and journals in related fields, including Nature, Science, Cell, Nature Machine Intelligence, ICML, ICLR, NeurIPS, etc.
\end{IEEEbiography}

\clearpage

\newpage
\appendices
\renewcommand{\thesection}{Appendix \Alph{section}}
\renewcommand{\thefigure}{A\arabic{figure}}
\renewcommand{\thetable}{A\arabic{table}}
\renewcommand{\theequation}{A\arabic{equation}}

\onecolumn

\crefalias{section}{appendix}

\section{Theoretical Interpretation and Analysis of MePo++}
\label{app:theory}

This appendix provides a theoretical interpretation of MePo++ from the perspective of sequential representation adaptation. Rather than claiming an exact probabilistic formulation or global convergence guarantee, our goal is to characterize several local properties that directly reflect the design of MetaPrep and StreamAlign. We first use sequential Bayesian updating as a conceptual decomposition of the plasticity--stability requirement in continual learning. We then show that the post-sequence objective underlying MetaPrep is locally sensitive to the compatibility between sequential adaptation directions and the broader pseudo-semantic objective. For StreamAlign, we establish the second-order geometry induced by its regularized covariance transport and show that representation reconciliation provides a controlled correction to the plastic online feature. Together, these results provide a method-specific interpretation of how MePo++ prepares representations before deployment and constrains their evolution during continual adaptation.

\subsection{A Sequential View of Continual Representation Adaptation}
\label{app:bayesian}

\myPara{Plasticity--Stability Decomposition}
Consider the GCL stream
$\mathcal{D}=\{\mathcal{B}_t\}_{t=1}^{T}$,
where $\mathcal{B}_t$ denotes the observations available at learning step $t$.
Let $\boldsymbol{\vartheta}$ denote the adaptive model state.
A generic sequential Bayesian update can be written as
\begin{equation}
p(\boldsymbol{\vartheta}\mid \mathcal{B}_{1:t})
=
\frac{
p(\mathcal{B}_t
\mid
\boldsymbol{\vartheta},
\mathcal{B}_{1:t-1})
p(\boldsymbol{\vartheta}\mid\mathcal{B}_{1:t-1})
}{
p(\mathcal{B}_t\mid\mathcal{B}_{1:t-1})
},
\label{eq:bayes_seq}
\end{equation}
where
$\mathcal{B}_{1:t}=\{\mathcal{B}_1,\ldots,\mathcal{B}_t\}$.
Taking the negative logarithm gives
\begin{equation}
-\log
p(\boldsymbol{\vartheta}\mid\mathcal{B}_{1:t})
=
\underbrace{
-\log
p(\mathcal{B}_t
\mid
\boldsymbol{\vartheta},
\mathcal{B}_{1:t-1})
}_{\mathcal{L}_{\mathrm{adapt}}^{t}}
+
\underbrace{
-\log
p(\boldsymbol{\vartheta}\mid\mathcal{B}_{1:t-1})
}_{\mathcal{R}_{\mathrm{retain}}^{t}}
+
C_t,
\label{eq:bayes_objective}
\end{equation}
where $C_t$ is independent of $\boldsymbol{\vartheta}$.
The first term favors responsiveness to current observations, corresponding conceptually to \emph{plasticity}, whereas the second preserves information accumulated from previous observations, corresponding to \emph{stability}.

Equation~\eqref{eq:bayes_objective} is used only as a conceptual decomposition and does not imply that MePo++ performs Bayesian inference. Its relevance to GCL is that both terms are difficult to realize directly: the incoming batch is sparse and potentially contains mixed concepts, while explicit task identities, boundaries, and complete historical statistics are unavailable. MePo++ therefore addresses these complementary requirements at different stages of the representation lifecycle rather than explicitly maintaining the posterior in Eq.~\eqref{eq:bayes_seq} during online adaptation.

\myPara{Lifecycle-Wise Decomposition}
Let $f_{\boldsymbol{\theta}_0}$ denote the original pretrained encoder.
Before downstream deployment, MetaPrep produces the following representation interface:
\begin{equation}
\operatorname{MetaPrep}
(\boldsymbol{\theta}_0,\mathcal{D}_{\mathrm{pre}})
\longrightarrow
(\boldsymbol{\theta}^{\ast},
\boldsymbol{\Sigma}_{\mathrm{pre}}),
\label{eq:theory_metaprep_interface}
\end{equation}
where $\boldsymbol{\theta}^{\ast}$ is refined according to its behavior after simulated sequential adaptation, while
$\boldsymbol{\Sigma}_{\mathrm{pre}}$
summarizes the global between-prototype geometry of the refined upstream representation.

During downstream learning, StreamAlign uses
\begin{equation}
\operatorname{StreamAlign}
(
\mathbf f_i^t,
\boldsymbol{\Sigma}_{\mathrm{cur}}^t,
\boldsymbol{\Sigma}_{\mathrm{pre}}
)
\longrightarrow
\mathbf f_{\mathrm{trans},i}^t
\label{eq:theory_streamalign_interface}
\end{equation}
to reconcile the evolving online representation with this fixed structural reference.
Importantly,
$\boldsymbol{\Sigma}_{\mathrm{pre}}$
is not intended to approximate a complete Bayesian prior or posterior.
It retains only second-order information about the organization of refined pseudo-semantic prototypes and therefore serves as a structural prior over representation geometry.

\subsection{Local Interpretation of MetaPrep for Continual Learnability}
\label{app:metaprep_theory}

We next examine why evaluating a representation \emph{after} pseudo-sequential adaptation provides a different optimization signal from ordinary post-training on shuffled upstream observations.

\myPara{First-Order Effect of Sequential Adaptation}
Consider one meta-epoch $k$ and suppress the superscript $(k)$ when no ambiguity arises.
Because both the encoder and auxiliary prediction head are updated in the inner loop, define their joint state for the sequential analysis as follows:
\begin{equation}
\mathbf u
=
(\boldsymbol{\theta},\boldsymbol{\psi}),
\qquad
\mathbf u_0
=
(\boldsymbol{\theta}_0,\boldsymbol{\psi}_0).
\end{equation}
Let
$\eta_{\psi}=\rho\eta_{\theta}$
for a fixed finite ratio $\rho>0$, and define
\begin{equation}
\mathbf P
=
\operatorname{diag}
(\mathbf I_{\theta},
\rho\mathbf I_{\psi}).
\end{equation}
The inner update on pseudo task $t$ can then be written compactly as
\begin{equation}
\mathbf u_t
=
\mathbf u_{t-1}
-
\eta_{\theta}
\mathbf P
\nabla_{\mathbf u}
\mathcal L_t(\mathbf u_{t-1}).
\label{eq:joint_inner_update}
\end{equation}

Define the pseudo-task gradient at the common expansion point as
\begin{equation}
\mathbf g_t
=
\nabla_{\mathbf u}
\mathcal L_t(\mathbf u_0).
\label{eq:task_grad}
\end{equation}
Assume that each $\mathcal L_t$ has locally Lipschitz-continuous gradients and that the inner-loop step sizes are sufficiently small.
A first-order expansion of the sequential trajectory gives
\begin{equation}
\mathbf u_{T'}
=
\mathbf u_0
-
\eta_{\theta}
\sum_{t=1}^{T'}
\mathbf P\mathbf g_t
+
\mathcal O(\eta_{\theta}^{2}).
\label{eq:seq_first_order}
\end{equation}
Thus, locally, the pseudo continual sequence perturbs the current initialization according to the aggregate adaptation directions induced by its constituent pseudo tasks.

\myPara{Post-Sequence Joint Objective}
Let
\begin{equation}
\mathbf g_{\mathrm{joint}}
=
\nabla_{\mathbf u}
\mathcal L_{\mathrm{joint}}(\mathbf u_0)
\label{eq:joint_grad}
\end{equation}
denote the gradient of the held-out joint objective at the same expansion point.

\begin{proposition}[Post-sequence compatibility]
\label{prop:metaprep}
Under the local smoothness and small-step assumptions above, the held-out joint objective after pseudo-sequential adaptation satisfies
\begin{equation}
\mathcal L_{\mathrm{joint}}
(\mathbf u_{T'})
=
\mathcal L_{\mathrm{joint}}
(\mathbf u_0)
-
\eta_{\theta}
\sum_{t=1}^{T'}
\left\langle
\mathbf g_{\mathrm{joint}},
\mathbf P\mathbf g_t
\right\rangle
+
\mathcal O(\eta_{\theta}^{2}).
\label{eq:meta_joint_expand}
\end{equation}
\end{proposition}

\noindent\textit{Proof.}
A first-order Taylor expansion around $\mathbf u_0$ gives
\begin{equation}
\mathcal L_{\mathrm{joint}}
(\mathbf u_{T'})
=
\mathcal L_{\mathrm{joint}}
(\mathbf u_0)
+
\left\langle
\mathbf g_{\mathrm{joint}},
\mathbf u_{T'}-\mathbf u_0
\right\rangle
+
\mathcal O
\left(
\|\mathbf u_{T'}-\mathbf u_0\|_2^2
\right).
\end{equation}
Substituting Eq.~\eqref{eq:seq_first_order} yields
Eq.~\eqref{eq:meta_joint_expand}.
\hfill$\square$

Proposition~\ref{prop:metaprep} characterizes the local effect of the pseudo-sequential trajectory underlying MetaPrep.
A pseudo-task update contributes to decreasing the post-sequence joint objective to first order when
\begin{equation}
\left\langle
\mathbf g_{\mathrm{joint}},
\mathbf P\mathbf g_t
\right\rangle
>
0.
\label{eq:positive_compatibility}
\end{equation}
Such an update is locally compatible with a descent direction that remains beneficial over the broader pseudo-semantic space.
Conversely, a negative inner product indicates local conflict between adaptation to the current pseudo task and performance on the joint pseudo-semantic objective.

This result does not imply that MetaPrep eliminates forgetting or explicitly enforces positive gradient agreement between all pseudo tasks.
More specifically, the Reptile-style meta-update used by MetaPrep does not directly optimize the inner products in Eq.~\eqref{eq:positive_compatibility}.
Rather, Proposition~\ref{prop:metaprep} shows that the \emph{post-sequence objective used to construct the meta-update} is locally sensitive to the compatibility between sequential adaptation directions and the broader pseudo-semantic objective, which provides a distinct training signal from ordinary static post-training on shuffled upstream data.

\myPara{Relation to Cross-Task Gradient Agreement}
A more explicit pairwise interpretation follows under an additional approximation.
Because
$\mathcal D_{\mathrm{joint}}$
contains held-out observations from all discovered pseudo concepts, suppose its local gradient can be approximated by the following weighted combination:
\begin{equation}
\mathbf g_{\mathrm{joint}}
\approx
\sum_{s=1}^{T'}
\omega_s\mathbf g_s,
\qquad
\omega_s\geq0,
\qquad
\sum_{s=1}^{T'}\omega_s=1.
\label{eq:joint_gradient_approx}
\end{equation}
This approximation requires the joint set to sufficiently represent the discovered pseudo-semantic space and does not hold as an identity in general for finite sampled gradients.

Substituting Eq.~\eqref{eq:joint_gradient_approx} into
Eq.~\eqref{eq:meta_joint_expand} gives
\begin{equation}
\mathcal L_{\mathrm{joint}}
(\mathbf u_{T'})
\approx
\mathcal L_{\mathrm{joint}}
(\mathbf u_0)
-
\eta_{\theta}
\sum_{s=1}^{T'}
\sum_{t=1}^{T'}
\omega_s
\left\langle
\mathbf g_s,
\mathbf P\mathbf g_t
\right\rangle
+
\mathcal O(\eta_{\theta}^{2}).
\label{eq:pairwise_gradient}
\end{equation}
For approximately balanced pseudo tasks,
$\omega_s=1/T'$,
yielding
\begin{equation}
\mathcal L_{\mathrm{joint}}
(\mathbf u_{T'})
\approx
\mathcal L_{\mathrm{joint}}
(\mathbf u_0)
-
\frac{\eta_{\theta}}{T'}
\sum_{s=1}^{T'}
\sum_{t=1}^{T'}
\left\langle
\mathbf g_s,
\mathbf P\mathbf g_t
\right\rangle
+
\mathcal O(\eta_{\theta}^{2}).
\label{eq:pairwise_gradient_balanced}
\end{equation}

The diagonal terms describe descent on individual pseudo tasks, whereas the off-diagonal terms characterize local interactions between different adaptation directions.
Positive interactions correspond to locally compatible updates, while negative interactions indicate gradient conflict.
Therefore, under the approximation in Eq.~\eqref{eq:joint_gradient_approx}, better post-sequence joint performance is associated with regions in which sequential adaptation induces less destructive local interaction.

This provides a local interpretation of \emph{continual learnability}: MetaPrep does not merely seek a representation with low static upstream loss, but evaluates whether the representation remains effective after successive pseudo-concept updates during downstream learning.
The pseudo sequence is intended to expose the encoder to sequential interference rather than to exactly reproduce the downstream blurry-stream distribution.

\myPara{Role of the Meta-Update}
After processing the pseudo continual sequence, MetaPrep performs joint refinement of the encoder using the held-out objective:
\begin{equation}
\hat{\boldsymbol{\theta}}_{T'}^{(k)}
=
\boldsymbol{\theta}_{T'}^{(k)}
-
\eta_{\theta}
\nabla_{\boldsymbol{\theta}}
\mathcal L_{\mathrm{joint}}^{(k)}
\left(
\boldsymbol{\theta}_{T'}^{(k)},
\boldsymbol{\psi}_{T'}^{(k)}
\right),
\label{eq:joint_refine_appendix}
\end{equation}
followed by the first-order meta-update
\begin{equation}
\boldsymbol{\theta}^{(k)}
=
\boldsymbol{\theta}^{(k-1)}
+
\eta_{\mathrm{meta}}
\left(
\hat{\boldsymbol{\theta}}_{T'}^{(k)}
-
\boldsymbol{\theta}^{(k-1)}
\right).
\label{eq:meta_outer_appendix}
\end{equation}
Hence, the encoder initialization is moved toward a parameter state obtained \emph{after} both sequential perturbation and joint evaluation.
This is fundamentally different from ordinary joint post-training, where the update is determined directly from shuffled upstream observations without any exposure to sequential interference.
Repeating the process across meta-epochs yields a GCL-oriented initialization
$\boldsymbol{\theta}^{\ast}$ for downstream continual adaptation under evolving observations.

\subsection{StreamAlign as Reference-Guided Geometry Transport}
\label{app:streamalign_theory}

At learning step $t$, let the online encoder produce row-vector representations
\begin{equation}
\mathbf f_i^t
=
f_{\boldsymbol{\theta}_t}(\mathbf x_i)
\in\mathbb R^d,
\end{equation}
and let
$\boldsymbol{\Sigma}_{\mathrm{cur}}^t$
denote their empirical covariance.
MetaPrep provides the fixed geometry prior
$\boldsymbol{\Sigma}_{\mathrm{pre}}$.
Because the feature dimension can exceed both the online batch size and the number of pseudo-class prototypes, these empirical covariance matrices can be rank deficient.
We therefore analyze the regularized matrices used for numerical transport:
\begin{equation}
\widetilde{\boldsymbol{\Sigma}}_{\mathrm{cur}}^t
=
\boldsymbol{\Sigma}_{\mathrm{cur}}^t
+
\epsilon\mathbf I,
\qquad
\widetilde{\boldsymbol{\Sigma}}_{\mathrm{pre}}
=
\boldsymbol{\Sigma}_{\mathrm{pre}}
+
\epsilon\mathbf I,
\qquad
\epsilon>0.
\label{eq:regularized_cov}
\end{equation}

\myPara{Reference-Guided Second-Order Geometry Transport}
Since the regularized matrices are positive definite, let
\begin{equation}
\widetilde{\boldsymbol{\Sigma}}_{\mathrm{cur}}^t
=
\mathbf L_{\mathrm{cur}}^t
(\mathbf L_{\mathrm{cur}}^t)^{\top},
\qquad
\widetilde{\boldsymbol{\Sigma}}_{\mathrm{pre}}
=
\mathbf L_{\mathrm{pre}}
\mathbf L_{\mathrm{pre}}^{\top}
\label{eq:chol_appendix}
\end{equation}
be their Cholesky decompositions.
Under the row-vector convention adopted in the main paper, define
\begin{equation}
\mathbf A_t
=
 (\mathbf L_{\mathrm{cur}}^t)^{-1}
\mathbf L_{\mathrm{pre}},
\qquad
\hat{\mathbf f}_i^t
=
\mathbf f_i^t\mathbf A_t.
\label{eq:transport_appendix}
\end{equation}

The transform is constructed from the Cholesky factors of the transient and reference covariance matrices and therefore depends explicitly on their second-order discrepancy. Under this row-vector convention, it should be interpreted as a reference-guided transport toward the upstream geometry rather than as an exact covariance-matching operator. The regularization in Eq.~\eqref{eq:regularized_cov} ensures numerical stability and helps avoid Cholesky failure when the empirical covariance is rank deficient, as can occur when the batch or prototype count is smaller than the feature dimension. This second-order correction alone does not imply that individual online features recover their upstream values, that the complete downstream distribution becomes identical to the upstream distribution, or that class-level discriminability is automatically preserved.

\myPara{Interpretation of the Stable Geometry Prior}
After MetaPrep, the structural prior used by StreamAlign is computed from the $M$ refined pseudo-class prototypes:
\begin{equation}
\boldsymbol{\Sigma}_{\mathrm{pre}}
=
\frac{1}{M-1}
\sum_{m=1}^{M}
(\boldsymbol{\mu}_m-\bar{\boldsymbol{\mu}})
(\boldsymbol{\mu}_m-\bar{\boldsymbol{\mu}})^{\top}.
\label{eq:pre_cov_appendix}
\end{equation}
Each prototype averages representations assigned to one pseudo class, so
Eq.~\eqref{eq:pre_cov_appendix}
emphasizes variation among pseudo-semantic centers while suppressing within-cluster sample variation.
In contrast,
$\boldsymbol{\Sigma}_{\mathrm{cur}}^t$
is computed from the current online mini-batch and can vary strongly with sparse, temporally mixed observations.

The two covariance matrices therefore have different statistical roles.
$\boldsymbol{\Sigma}_{\mathrm{pre}}$
is not intended to be an unbiased estimator of the instantaneous downstream covariance.
Instead, it serves as a \emph{target geometry} encoding the broader organization of the refined upstream pseudo-semantic space.
We call it stable in an operational sense because it is estimated once from the upstream reference set and kept fixed throughout downstream adaptation, whereas
$\boldsymbol{\Sigma}_{\mathrm{cur}}^t$
changes with every incoming batch.
Accordingly, the prior can become less representative under severe or long-term distribution shift, as discussed in the limitations of the main paper and its deployment assumptions over extended horizons.

\subsection{Controlled Representation Reconciliation}
\label{app:reconciliation_theory}

Directly replacing the online representation
$\mathbf f_i^t$
with its reference-aligned counterpart
$\hat{\mathbf f}_i^t$
would maximize the influence of the upstream geometry but may suppress useful stream-specific changes.
StreamAlign therefore interpolates between the two:
\begin{equation}
\mathbf f_{\mathrm{trans},i}^{t}
=
(1-\alpha)\mathbf f_i^t
+
\alpha\hat{\mathbf f}_i^t,
\qquad
\alpha\in[0,1].
\label{eq:interpolate_appendix}
\end{equation}

\begin{proposition}[Bounded representation correction]
\label{prop:bounded_reconciliation}
For any
$\mathbf f_i^t$,
$\hat{\mathbf f}_i^t$,
and $\alpha\in[0,1]$,
the reconciled representation satisfies
\begin{equation}
\left\|
\mathbf f_{\mathrm{trans},i}^{t}
-
\mathbf f_i^t
\right\|_2
=
\alpha
\left\|
\hat{\mathbf f}_i^t
-
\mathbf f_i^t
\right\|_2,
\label{eq:bounded_online}
\end{equation}
and
\begin{equation}
\left\|
\mathbf f_{\mathrm{trans},i}^{t}
-
\hat{\mathbf f}_i^t
\right\|_2
=
(1-\alpha)
\left\|
\hat{\mathbf f}_i^t
-
\mathbf f_i^t
\right\|_2.
\label{eq:bounded_reference}
\end{equation}
\end{proposition}

\noindent\textit{Proof.}
From Eq.~\eqref{eq:interpolate_appendix}, the following identity holds:
\begin{equation}
\mathbf f_{\mathrm{trans},i}^{t}
-
\mathbf f_i^t
=
\alpha
(\hat{\mathbf f}_i^t-\mathbf f_i^t),
\end{equation}
which immediately gives
Eq.~\eqref{eq:bounded_online}.
Similarly,
\begin{equation}
\mathbf f_{\mathrm{trans},i}^{t}
-
\hat{\mathbf f}_i^t
=
(1-\alpha)
(\mathbf f_i^t-\hat{\mathbf f}_i^t),
\end{equation}
which directly gives the stated relation in
Eq.~\eqref{eq:bounded_reference}.
\hfill$\square$

Proposition~\ref{prop:bounded_reconciliation} establishes a controlled interpolation property rather than an information-preservation guarantee.
When $\alpha=0$, the online representation is left unchanged.
When $\alpha=1$, the reconciled feature becomes fully reference-aligned.
For $0<\alpha<1$, the deviation from the plastic online representation is exactly an $\alpha$ fraction of the full reconstruction displacement.
Thus, StreamAlign introduces a controllable structural correction instead of projecting the online feature completely back to the upstream geometry during interpolation.

This property also provides a local explanation for the empirical behavior of the reconciliation weight.
Although the reference-guided transport characterizes the desired second-order direction of the reconstructed representation, it does not guarantee preservation of all stream-specific discriminative information.
Aggressive alignment can therefore suppress useful online adaptation, whereas intermediate values retain a direct contribution from the plastic representation.

\myPara{Why Semantic Reconciliation is Necessary}
Second-order geometry does not uniquely determine semantic organization.
Different class configurations can share the same global covariance matrix.
Therefore, even exact covariance matching would not imply that same-class samples remain compact or that different classes remain separated.

StreamAlign addresses this ambiguity using the supervised semantic reconciliation objective.
For normalized representations
$\mathbf z_a$
and
$\mathbf z_p$,
the positive-pair term is
\begin{equation}
\ell_{a,p}^{t}
=
-\log
\frac{
\exp(\mathbf z_a^{\top}\mathbf z_p/\tau)
}{
\sum_{q\in\mathcal V_t\setminus\{a\}}
\exp(\mathbf z_a^{\top}\mathbf z_q/\tau)
}.
\label{eq:contrastive_appendix}
\end{equation}
For fixed competing similarities, increasing the positive similarity
$\mathbf z_a^{\top}\mathbf z_p$
decreases this objective, whereas increasing similarity to competing representations enlarges the denominator.
Averaging over same-label positives therefore encourages semantic consistency across the plastic and reference-aligned representation spaces while discouraging collapse across different classes.

This construction is also compatible with sparse online batches encountered during downstream adaptation.
Because each observed sample contributes both an original view
$\mathbf f_i^t$
and a reconstructed view
$\hat{\mathbf f}_i^t$,
the two views provide a same-label positive pair even when that class appears only once in the current mini-batch.
Nevertheless, the contrastive objective does not guarantee a fixed classification margin.
It supplies semantic supervision that is absent from covariance matching alone, encouraging cross-view consistency at the class level.

The complete downstream objective
\begin{equation}
\mathcal L_{\mathrm{GCL}}^t
=
\mathcal L_{\mathrm{main}}^t
+
\lambda
\mathcal L_{\mathrm{con}}^t
\label{eq:gcl_appendix}
\end{equation}
therefore combines complementary signals:
$\mathcal L_{\mathrm{main}}^t$
keeps the learner responsive to current observations, while
$\mathcal L_{\mathrm{con}}^t$
encourages semantic compatibility between the evolving and reference-aligned representation spaces.
Geometric reconstruction supplies a stable global structural reference, whereas semantic reconciliation protects the class-level organization required for downstream prediction throughout online adaptation.

\subsection{Unified Interpretation}
\label{app:unified_theory}

The above analyses characterize complementary properties of the two stages of MePo++.

For MetaPrep, Proposition~\ref{prop:metaprep} shows that the held-out objective after pseudo-sequential adaptation is locally determined by the compatibility between sequential adaptation directions and the broader pseudo-semantic descent direction.
Under the additional approximation in Eq.~\eqref{eq:joint_gradient_approx},
Eqs.~\eqref{eq:pairwise_gradient}--\eqref{eq:pairwise_gradient_balanced}
further expose local interactions among pseudo-task gradients.
These results provide a mechanism through which MetaPrep can favor initializations that are less susceptible to destructive interference across sequential updates and evolving pseudo-concept streams.

For StreamAlign, the reference-guided geometry transport constructs each reconstruction from the transient and stable covariance factors, while Proposition~\ref{prop:bounded_reconciliation} shows that the influence of this structural reference on each online representation is explicitly controlled by $\alpha$ across the continual stream.
The semantic objective complements these geometric properties because global covariance alone does not determine class-level discriminability during continual adaptation.

Taken together, the two components act on different stages of continual representation adaptation:
\begin{equation}
\boxed{
\begin{aligned}
\bigl(\boldsymbol{\theta}_0,\mathcal{D}_{\mathrm{pre}}\bigr)
&\xrightarrow{\ \mathrm{MetaPrep}\ }
\bigl(\boldsymbol{\theta}^{\ast},\boldsymbol{\Sigma}_{\mathrm{pre}}\bigr),
\\
\bigl(\mathbf f_i^t,\boldsymbol{\Sigma}_{\mathrm{cur}}^t,
\boldsymbol{\Sigma}_{\mathrm{pre}}\bigr)
&\xrightarrow{\ \mathrm{StreamAlign}(\alpha)\ }
\mathbf f_{\mathrm{trans},i}^{t}.
\end{aligned}}
\label{eq:unified_interpretation}
\end{equation}

MetaPrep changes the \emph{starting representation} from which downstream sequential adaptation proceeds, whereas StreamAlign constrains the \emph{trajectory of the evolving representation} after deployment.
Their roles are therefore complementary rather than redundant:
the former prepares the representation to accommodate future updates, while the latter prevents transient online observations from unconstrainedly determining the global representation geometry.

\end{document}